\documentclass{article}

\usepackage[final,nonatbib]{nips_2017_arxiv}

\usepackage[utf8]{inputenc} 
\usepackage[T1]{fontenc}    
\usepackage{hyperref}       
\usepackage{url}            
\usepackage{booktabs}       
\usepackage{amsfonts}       
\usepackage{nicefrac}       
\usepackage{microtype}      
\usepackage{makecell}
\usepackage{amsmath,amsfonts,graphpap,amscd,mathrsfs,graphicx,lscape}
\usepackage{epsfig,amssymb,amstext,xspace}
\usepackage{float}	
\usepackage{hyperref}

\usepackage{amsthm}
\usepackage{array}
\usepackage{float}
\usepackage{setspace}
\newcolumntype{C}[1]{>{\centering\let\newline\\\arraybackslash\hspace{0pt}}m{#1}}

\usepackage{algorithmic}
\usepackage{algorithm}

\newtheorem{theorem}{Theorem}

\newtheorem{corollary}[theorem]{Corollary}

\newtheorem{defn}{Definition}

\usepackage{color}              
\usepackage{epsfig}

\usepackage[]{color-edits}
\addauthor{bl}{blue}
\addauthor{vs}{red}

\newcommand{\R}{\ensuremath{\mathbb R}}

\newcommand{\F}{\ensuremath{\mathcal F}}

\newcommand{\D}{\ensuremath{\mathcal D}}
\newcommand{\X}{\ensuremath{{\cal X}}}
\newcommand{\G}{\ensuremath{{\cal G}}}

\newcommand{\Ell}{\ensuremath{\mathcal{L}}}
\newcommand{\Hy}{\ensuremath{\mathcal{H}}}
\newcommand{\Z}{\ensuremath{\mathcal{Z}}}
\newcommand{\Y}{\ensuremath{\mathcal{Y}}}
\newcommand{\W}{\ensuremath{\mathcal{W}}}

\newcommand{\E}{\mathbb{E}}

\newcommand{\Po}{\ensuremath{\mathcal{P}}}

\renewcommand{\vec}[1]{\ensuremath{\textbf{#1}}}

\newcommand{\mytitle}{Robust Optimization for Non-Convex Objectives}
\title{\mytitle}

\author{
  Robert Chen\\
  Computer Science\\
  Harvard University
  \And
  Brendan Lucier\\
  Microsoft Research\\
  New England
  \And 
  Yaron Singer\\
  Computer Science\\
  Harvard University
  \And
  Vasilis Syrgkanis\\
  Microsoft Research\\
  New England
}

\begin{document}

\maketitle

\begin{abstract}
We consider robust optimization problems, where the goal is to optimize in the worst case over a class of objective functions.  We develop a reduction from robust improper optimization to Bayesian optimization: given an oracle that returns $\alpha$-approximate solutions for distributions over objectives, we compute a distribution over solutions that is $\alpha$-approximate in the worst case.  We show that derandomizing this solution is NP-hard in general, but can be done for a broad class of statistical learning tasks.  We apply our results to robust neural network training and submodular optimization.  We evaluate our approach experimentally on corrupted character classification, and robust influence maximization in networks.


\end{abstract}

\section{Introduction}






In many learning tasks we face uncertainty about the loss we aim to optimize.
Consider, for example, a classification task such as character recognition, required to perform well under various types of distortion.  In some environments, such as recognizing characters in photos, the classifier must handle rotation and patterned backgrounds.  In a different environment, such as low-resolution images, it is more likely to encounter noisy pixelation artifacts.  Instead of training a separate classifier for each possible scenario, one seeks to optimize performance 
in the worst case over different forms of corruption (or combinations thereof) made available to the trainer as black-boxes.  


More generally, our goal is to find a minimax solution that optimizes in the worst case over a given family of functions.  
Even if each individual function can be optimized effectively, it is not clear such solutions would perform well in the worst case.  In many cases of interest, individual objectives are non-convex and hence state-of-the-art methods are only approximate.  In Bayesian optimization, where one must optimize a distribution over loss functions,  approximate Bayesian optimization is often straightforward, 
since loss functions are commonly closed under convex combination.  Can approximately optimal solutions yield an approximately optimal \emph{robust} solution?  

In this paper we develop a reduction from robust optimization to Bayesian optimization.  Given an $\alpha$-approximate oracle for Bayesian optimization we show how to implement an $\alpha$-approximate solution for robust optimization under a necessary extension, and illustrate its effectiveness in applications.

\paragraph{Main Results.}
Given an $\alpha$-approximate Bayesian oracle for distributions over (potentially non-convex) loss functions, we show how to solve $\alpha$-approximate robust optimization in a convexified solution space.  This outcome is ``improper'' in the sense that it may lie outside the original solution space, if the space is non-convex.   This can be interpreted as computing a distribution over solutions.
We show that the relaxation to improper learning is necessary in general:  It is NP-hard to achieve robust optimization with respect to the original outcome space, even if Bayesian optimization can be solved exactly, and even if there are only polynomially many loss functions.  
We complement this by showing that in any statistical learning scenario where loss is convex in the predicted dependent variable, we can find a single (deterministic) solution with matching performance guarantees.

\paragraph{Technical overview.}  Our approach employs an execution of no-regret dynamics on a zero-sum game, played between a learner equipped with an $\alpha$-approximate Bayesian oracle, and an adversary who aims to find a distribution over loss functions that maximizes the learner's loss.  This game converges to an approximately robust solution, in which the learner and adversary settle upon an $\alpha$-approximate minimax solution.  This convergence is subject to an additive regret term that converges at a rate of $T^{-1/2}$ over $T$ rounds of the learning dynamics.

\paragraph{Applications.}
We illustrate the power of our reduction through two main examples.  We first consider statistical learning via neural networks.  Given an arbitrary training method, our reduction generates a net that optimizes robustly over a given class of loss functions.  We evaluate our method experimentally on a character recognition task, where the loss functions correspond to different corruption models made available to the learner as black boxes.  We verify experimentally that our approach significantly outperforms various baselines, including optimizing for average performance and optimizing for each loss separately.  We also apply our reduction to influence maximization, where the goal is to maximize a concave function (the independent cascade model of influence \cite{KempeKT03}) over a non-convex space (subsets of vertices in a network).  Previous work has studied robust influence maximization directly \cite{HeKempe16,ChenLTZZ16,LowalekarVK16}, focusing on particular, natural classes of functions (e.g., edge weights chosen within a given range) and establishing hardness and approximation results.  In comparison, our method is agnostic to the particular class of functions, and achieves a strong approximation result by returning a distribution over solutions.  We evaluate our method on real and synthetic datasets, with the goal of robustly optimizing a suite of random influence instantiations.  We verify experimentally that our approach significantly outperforms natural baselines.

\paragraph{Related work.}  
There has recently been a great deal of interest in robust optimization in machine learning~\cite{SW16,CDLZ16,ND16,SD15}.  For continuous optimization, the work that is closest to ours is perhaps that by Shalev-Shwartz and Wexler~\cite{SW16} and Namkoong and Duchi~\cite{ND16} that use robust optimization to train against convex loss functions.  The main difference is that we assume a more general setting in which the loss functions are non-convex and one is only given access to the Bayesian oracle.  Hence, the proof techniques and general results from these papers do not apply to our setting.  We note that our result generalizes these works, as they can be considered as the special case in which we have a distributional oracle whose approximation is optimal.  In submodular optimization there has been a great deal of interest in robust optimization as well~\cite{KMGG07,HK16,CLTZZ16}.  The work closest to ours is that by He and Kempe~\cite{HK16} who consider a slightly different objective than ours.  
Kempe and He's results apply to influence but do not extend to general submodular functions.  Finally, we note that unlike recent work on non-convex optimization~\cite{HLS15,ZH16,HLS16} our goal in this paper is not to optimize a non-convex function.  Rather, we abstract the non-convex guarantees via the approximate Bayesian oracle.

\section{Robust Optimization with Approximate Bayesian Oracles}

We consider the following model of optimization that is robust to objective uncertainty. There is a space $\X$ over which to optimize, and a finite set of loss functions\footnote{We describe an extension to infinite sets of loss functions in Appendix~\ref{app:infinite}.  Our results also extend naturally to the goal of maximizing the minimum of a class of reward functions.} $\Ell=\{L_1,\ldots,L_m\}$ where each $L_i \in \Ell$ is a function from $\X$ to $[0,1]$.
%
%
Intuitively, our goal is to find some $x\in \X$ that achieves low loss in the worst-case over loss functions in $\Ell$. For $x \in \X$, write $g(x) = \max_{i\in [m]} L_i(x)$ for the worst-case loss of $x$. The minimax optimum $\tau$ is given by
\begin{equation}
\tau = \min_{x\in \X}g(x) = \min_{x\in \X}\max_{i\in [m]} L_i(x).
\end{equation}
The goal of $\alpha$-approximate robust optimization is to find $x$ such that $g(x) \leq \alpha \tau$.

Given a distribution $\Po$ over solutions $\X$, write $g(\Po) = \max_{i \in [m]}\E_{x \sim \Po}[L_i(x)]$ for the worst-case expected loss of a solution drawn from $\Po$.  A weaker version of robust approximation is \emph{improper robust optimization}: find a distribution $\Po$ over $\X$ such that $g(\Po) \leq \alpha \tau$.

Our results take the form of reductions to an approximate Bayesian oracle, which finds a solution $x \in \X$ that approximately minimizes a given distribution over loss functions.\footnote{All our results easily extend to the case where the oracle computes a solution that is approximately optimal up to an additive error, rather than a multiplicative one. For simplicity of exposition we present the multiplicative error case as it is more in line with the literature on approximation algorithms.} 
\begin{defn}[$\alpha$-Approximate Bayesian Oracle]\label{defn:oracle} Given a distribution $D$ over $\Ell$, an $\alpha$-approximate Bayesian Oracle $M(D)$ computes 
$x^* \in \X$ such that
\begin{equation}
\E_{L \sim D}\left[ L(x^*)\right] \leq \alpha\min_{x\in \X} \E_{L \sim D}\left[ L(x)\right].
\end{equation}
\end{defn}


%

\subsection{Improper Robust Optimization with Oracles}

We first show that, given access to an $\alpha$-approximate distributional oracle, it is possible to efficiently implement improper $\alpha$-approximate robust optimization, subject to a vanishing additive loss term.


\begin{algorithm}[t]
\caption{Oracle Efficient Improper Robust Optimization} 
\label{alg:main}
\begin{algorithmic}
	\STATE {\bfseries Input:} Objectives $\Ell=\{L_1,\ldots,L_m\}$, Apx Bayesian oracle $M$, parameters $T,\eta$ 
	\FOR{each time step $t\in [T]$}
   	\STATE Set  
	\begin{equation}
	\vec{w}_t[i] \propto \exp\left\{ \eta\sum_{\tau=1}^{t-1} L_i(x_	\tau)\right\} 
	\end{equation}   	
   	\STATE Set $x_t = M(\vec{w}_t)$
	\ENDFOR
	\STATE \textbf{Output:} the uniform distribution over $\{x_1,\ldots,x_T\}$
\end{algorithmic}
\end{algorithm}

\begin{theorem}\label{thm:distributionally-robust}
Given access to an $\alpha$-approximate distributional oracle, Algorithm \ref{alg:main} with $\eta=\sqrt{\frac{\log(m)}{2T}}$ computes a distribution $\Po$ over solutions, defined as a uniform distribution over a set $\{x_1,\ldots,x_T\}$, so that
\begin{equation}
\max_{i\in [m]} \E_{x\sim \Po}\left[L_i(x)\right] \leq \alpha \tau + \sqrt{\frac{2\log(m)}{T}}.
\end{equation}
Moreover, for any $\eta$ the distribution $\Po$ computed by Algorithm \ref{alg:main} satisfies:
\begin{equation}
\max_{i\in [m]} \E_{x\sim \Po}\left[L_i(x)\right] \leq \alpha(1+\eta)\tau + \frac{2\log(m)}{\eta T}.
\end{equation}
\end{theorem} 
\begin{proof}
We give the proof of the first result and defer the second result to Theorem \ref{thm:faster} in Appendix \ref{app:faster}. We can interpret Algorithm~\ref{alg:main} in the following way.  We define a zero-sum game between a learner and an adversary. The learner's action set is equal to $\X$ and the adversary's action set is equal to $[m]$. The loss of the learner when he picks $x\in \X$ and the adversary picks $i\in [m]$ is defined as $L_i(x)$. The corresponding payoff of the adversary is $L_i(x)$.

We will run no-regret dynamics on this zero-sum game, where at every iteration $t=1,\ldots,T$, the adversary will pick a distribution over functions and subsequently the learner picks a solution $x_t$. For simpler notation we will denote with $\vec{w}_t$ the probability density function on $[m]$ associated with the distribution of the adversary.  That is, $w_t[i]$ is the probability of picking function $L_i\in\Ell$. The adversary picks a distribution $\vec{w}_t$ based on some arbitrary no-regret learning algorithm on the $k$ actions in $\F$. For concreteness consider the case where the adversary picks a distribution based on the multiplicative weight updates algorithm, i.e.,
\begin{equation}
w_t[i] \propto \exp\left\{ \sqrt{\frac{\log(m)}{2T}}\sum_{\tau=1}^{t-1} L_i(x_\tau)\right\}.
\end{equation}
Subsequently the learner picks a solution $x_t$ that is the output of the $\alpha$-approximate distributional oracle on the distribution selected by the adversary at time-step $t$.  That is,
\begin{equation}
x_t = M\left(\vec{w}_t\right).
\end{equation}

Write $\epsilon(T) = \sqrt{\frac{2\log(m)}{T}}$.  By the guarantees of the no-regret algorithm  for the adversary, we have that
\begin{equation}
\frac{1}{T}\sum_{t=1}^T \E_{I\sim \vec{w}_t}\left[L_I(x_t)\right] \geq \max_{i\in [m]} \frac{1}{T}\sum_{t=1}^T L_i(x_t) - \epsilon(T).
\end{equation}
Combining the above with the guarantee of the distributional oracle we have
\begin{align*}
\tau = \min_{x\in \X}\max_{i\in [m]}  L_i(x) \geq~& \min_{x\in X}\frac{1}{T}\sum_{t=1}^T \E_{I\sim \vec{w}_t}\left[L_I(x)\right]
\geq~ \frac{1}{T}\sum_{t=1}^T \min_{x\in X}\E_{I\sim \vec{w}_t}\left[L_I(x)\right]\\
\geq~& \frac{1}{T}\sum_{t=1}^T \frac{1}{\alpha}\cdot \E_{I\sim \vec{w}_t}\left[L_I(x_t)\right] \tag{By oracle guarantee for each $t$}\\
\geq~& \frac{1}{\alpha}\cdot \left(\max_{i\in [m]} \frac{1}{T}\sum_{t=1}^T L_i(x_t) - \epsilon(T)\right).
\tag{By no-regret of adversary}
\end{align*}
Thus, if we define with $\Po$ to be the uniform distribution over $\{x_1,\ldots,x_T\}$, then we have derived
\begin{equation}
\max_{i\in [m]} \E_{x\sim \Po}\left[L_i(x) \right] \leq \alpha\tau + \epsilon(T)
\end{equation}
as required.
\end{proof}

A corollary of Theorem~\ref{thm:distributionally-robust} is that if the solution space $\X$ is convex and the objective functions $L_i\in \Ell$ are all convex functions, then we can compute a single solution $x^*$ that is approximately minimax optimal.  Of course, in this setting one can calculate and optimize the maximum loss directly in time proportional to $|\Ell|$; this result therefore has the most bite when the set of functions is large.
\begin{corollary}\label{cor:convex}
If the space $\X$ is a convex space and each loss function $L_i\in \Ell$ is a convex function, then the point $x^* = \frac{1}{T}\sum_{t=1}^T x_t \in \X$, where $\{x_1,\ldots,x_T\}$ are the output of Algorithm \ref{alg:main}, satisfies:
\begin{equation}
\max_{i\in [m]} L_i(x^*) \leq \alpha \tau + \sqrt{\frac{2\log(m)}{T}}
\end{equation}
\end{corollary}
\begin{proof}
By Theorem \ref{thm:distributionally-robust}, we get that if $\Po$ is the uniform distribution over $\{x_1,\ldots,x_T\}$ then
\begin{equation*}
\max_{i\in [m]} \E_{x\sim \Po}[L_i(x)] \leq \alpha \tau + \sqrt{\frac{2\log(m)}{T}}.
\end{equation*}
Since $\X$ is convex, the solution $x^* = \E_{x\sim \Po}[x]$ is also part of $\X$. Moreover, since each $L_i\in \Ell$ is convex, we have that $\E_{x\sim \Po}[L_i(x)]\geq L_i(\E_{x\sim \Po}[x]) = L_i(x^*)$. We therefore conclude
\begin{equation*}
\max_{i\in [m]} L_i(x^*) \leq \max_{i\in [m]} \E_{x\sim \Po}[L_i(x)] \leq \alpha \tau + \sqrt{\frac{2\log(m)}{T}}
\end{equation*}
as required.
\end{proof}
\subsection{Robust Statistical Learning}\label{sec:statistical}

Next we apply our main theorem to 
statistical learning. 
Consider regression or classification settings where data points are pairs $(z,y)$, $z\in \Z$ is a vector of features, and $y\in \Y$ is the dependent variable.  
The solution space $\X$ is then a space of hypotheses $\Hy$, with each $h\in \Hy$ a function from $\Z$ to $\Y$. We also assume that $\Y$ is a convex subset of a finite-dimensional vector space.

We are given a set of loss functions $\Ell=\{L_1,\ldots,L_m\}$, where each $L_i\in \Ell$ is a functional $L_i \colon \Hy \to [0,1]$. 
Theorem~\ref{thm:distributionally-robust} implies that, given an $\alpha$-approximate Bayesian optimization oracle, we can compute a distribution over $T$ hypotheses from $\Hy$ that achieves an $\alpha$-approximate minimax guarantee.  If the loss functionals are convex over hypotheses, then we can compute a single ensemble hypothesis $h^*$ (possibly from a larger space of hypotheses, if $\Hy$ is non-convex) that achieves this guarantee.  
%
\begin{theorem}\label{thm:improper-learning}
Suppose that $\Ell = \{L_1, \dotsc, L_m\}$ are convex functionals. Then the ensemble hypothesis $h^*=\frac{1}{T}\sum_{t=1}^T h$, where $\{h_1,\ldots,h_T\}$ are the hypotheses output by Algorithm \ref{alg:main} given an $\alpha$-approximate Bayesian oracle, satisfies
\begin{equation}
\max_{i\in [m]} L_i(h^*) \leq \alpha \min_{h\in H}\max_{i\in [m]} L_i(h) + \sqrt{\frac{2\log(m)}{T}}.
\end{equation}
\end{theorem}
\begin{proof}
The proof is similar to the proof of Corollary \ref{cor:convex}.
\end{proof}
 
We emphasize that the convexity condition in Theorem~\ref{thm:improper-learning} is over the class of hypotheses, rather than over features or any natural parameterization of $\Hy$ (such as weights in a neural network).  This is a mild condition that applies to many examples in statistical learning theory. For instance, consider the case where each loss $L_i(h)$ is the expected value of some ex-post loss function $\ell_i(h(z),y)$ given a distribution $D_i$ over $Z\times Y$:
\begin{equation}
L_{i}(h) = \E_{(z,y)\sim D_i}\left[\ell_i(h(z),y)\right].
\end{equation}
In this case, it is enough for
the function $\ell_i(\cdot,\cdot)$ to be convex with respect to its first argument (i.e., the predicted dependent variable).  This is satisfied by most loss functions used in machine learning, such as multinomial logistic loss (cross-entropy loss) $\ell(\hat{y},y)=-\sum_{c\in [k]} y_c \log(\hat{y}_c)$ from multi-class classification, the hinge or the square loss, or squared loss $\ell(\hat{y},y)=\|\hat{y}-y\|^2$ as used in regression.
For all these settings, Theorem \ref{thm:improper-learning} provides a tool for improper robust learning, where the final hypothesis $h^*$ is an ensemble of $T$ base hypotheses from $\Hy$.  Again, the underlying optimization problem can be arbitrarily non-convex in the natural parameters of the hypothesis space; in Section~\ref{sec:neural-nets} we will show how to apply this approach to robust training of neural networks, where the Bayesian oracle is simply a standard network training method.  For neural networks, the fact that we achieve improper learning (as opposed to standard learning) corresponds to training a neural network with a single extra layer relative to the networks generated by the oracle.


\subsection{Robust Submodular Maximization}\label{sec:submodular}

In \emph{robust submodular maximization} we are given a family of reward functions $\mathcal{F} = \{f_{1},\ldots,f_{m}\}$, where each $f_i\in \F$ is a monotone submodular function from a ground set $N$ of $n$ elements to $[0,1]$.  Each function is assumed to be monotone and submodular, i.e., for any $S\subseteq T\subseteq 2^N$, $f_i(S) \leq f_i(T)$; and for any $S,T\subseteq 2^N$, $f(S\cup T)+f(S\cap T)\leq f(S) + f(T)$. The goal is to select a set $S \subseteq N$ of size $k$ whose worst-case value over $i$, i.e., $g(S)=\min_{i\in [m]} f_i(S)$, is at least a $1/\alpha$ factor of the minimax optimum 
%
%
$\tau=\max_{T:|T|\leq k}\min_{i\in[m]} f_i(T)$.

This setting is a special case of our general robust optimization setting (phrased in terms of rewards rather than losses). The solution space $\X$ is equal to the set of subsets of size $k$ among all elements in $N$ and the set $\F$ is the set of possible objective functions. The Bayesian oracle \ref{defn:oracle}, instantiated in this setting, asks for the following: given a convex combination of submodular functions $F(S) = \sum_{i=1}^m \vec{w}[i]\cdot f_i(S)$, compute a set $S^*$ 
such that $F(S^*) \geq \frac{1}{\alpha}\max_{S: |S|\leq k} F(S)$. 

Computing the maximum value set of size $k$ is NP-hard even for a single submodular function.  The following very simple greedy algorithm computes a $(1-1/e)$-approximate solution~\cite{Nemhauser1978}:
begin with $S_{cur} = \emptyset$, and at each iteration add to the current solution $S_{cur}$ the element $j\in N-S_{cur}$ that has the largest marginal contribution: $f(\{j\}\cup S_{cur})-f(S_{cur})$. Moreover, this approximation ratio is known to be the best possible in polynomial time~\cite{NemhauserW78b}. Since a convex combination of monotone submodular functions is also a monotone submodular function, we immediately get that there exists a $(1-1/e)$-approximate Bayesian oracle that can be computed in polynomial time.  The algorithm is formally given in Algorithm \ref{alg:greedy}.
\begin{algorithm}[t]
\caption{Greedy Bayesian Oracle for Submodular Maximization $M_{greedy}$} 
\label{alg:greedy}
\begin{algorithmic}
	\STATE {\bfseries Input:} Set of elements $N$, objectives $\F=\{f_1,\ldots,f_m\}$, distribution over objectives $\vec{w}$
	\STATE Set $S_{cur} = \emptyset$
	\FOR{$j=1$ to $k$}
   	\STATE Let $j^* = \arg\max_{j\in N-S_{cur}} \sum_{i=1}^m \vec{w}[i] \left(f_i(\{j\}\cup S_{cur})-f_i(S_{cur})\right)$
   	\STATE Set $S_{cur}=\{j^*\}\cup S_{cur}$
	\ENDFOR
\end{algorithmic}
\end{algorithm}
Combining the above with Theorem \ref{thm:distributionally-robust} we get the following corollary.
\begin{corollary}
Algorithm \ref{alg:main}, with Bayesian oracle $M_{greedy}$, computes in time $poly(T,n)$ a distribution $\Po$ over sets of size $k$, defined as a uniform distribution over a set $\{S_1,\ldots,S_T\}$, such that
\begin{equation}
\min_{i \in [m]} \E_{S\sim \Po}\left[f_i(S)\right] \geq \left( 1-\frac{1}{e} \right)(1-\eta) \tau - \frac{\log(m)}{\eta T}.
\end{equation}
\end{corollary}

As we show in Appendix \ref{app:NP-hard}, computing a single set $S$ that achieves a $(1-1/e)$-approximation to $\tau$ is also $NP$-hard. This is true even if the functions $f_i$ are additive.  However, by allowing a randomized solution over sets we can achieve a constant factor approximation to $\tau$ in polynomial time. 

Since the functions are monotone, the above result implies a simple way of constructing a single set $S^*$ that is of larger size than $k$, which deterministically achieves a constant factor approximation to $\tau$. The latter holds by simply taking the union of the sets $\{S_1,\ldots,S_T\}$ in the support of the distribution returned by Algorithm~\ref{alg:main}.  We get the following bi-criterion approximation scheme. 
\begin{corollary}
Suppose that we run the reward version of Algorithm \ref{alg:main}, with $\eta=\epsilon$ and for $T=\frac{\log(m)}{\tau \epsilon^2}$, returning $\{S_1,\ldots, S_T\}$. Then the set $S^*=S_1\cup\ldots\cup S_T$, which is of size at most $\frac{k\log(m)}{\tau\epsilon^2}$, satisfies
\begin{equation}
\min_{i\in [m]} f_i(S^*) \geq \left(1-\frac{1}{e}-2\epsilon\right)\tau.
\end{equation} 
\end{corollary}

\begin{figure}[t]
\begin{center}
    \includegraphics[scale = .12]{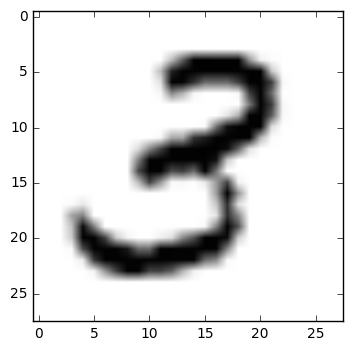}
    \includegraphics[scale = .12]{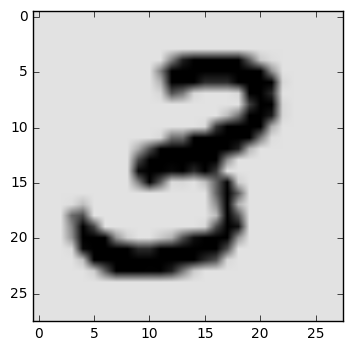}
    \includegraphics[scale = .12]{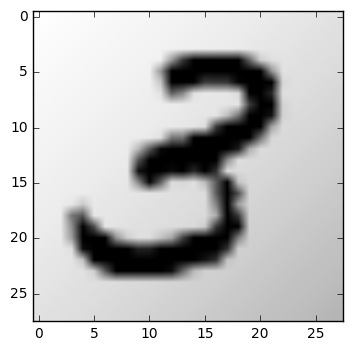}
    \includegraphics[scale = .12]{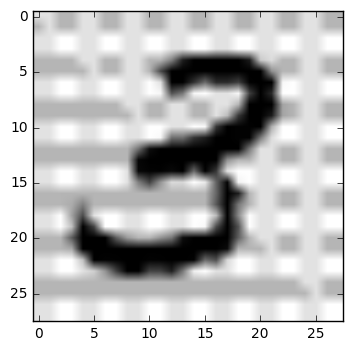}
    \hspace{0.5cm}
    \includegraphics[scale = .12]{three_none.png}
    \includegraphics[scale = .12]{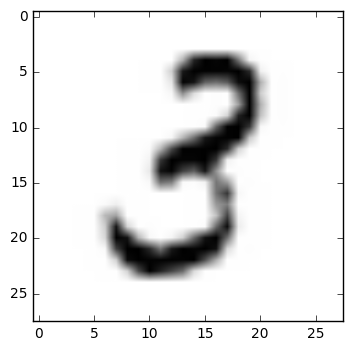}
    \includegraphics[scale = .12]{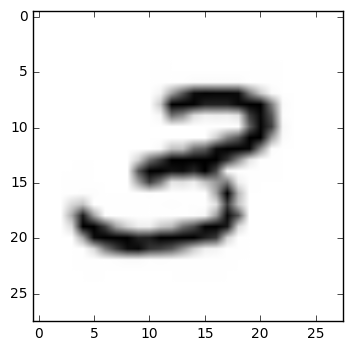}
    \includegraphics[scale = .12]{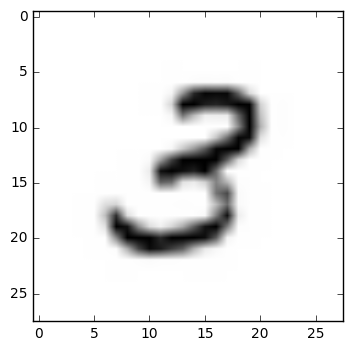}\\
    \vspace{0.5cm}
    \includegraphics[scale = .12]{three_none.png}
    \includegraphics[scale = .12]{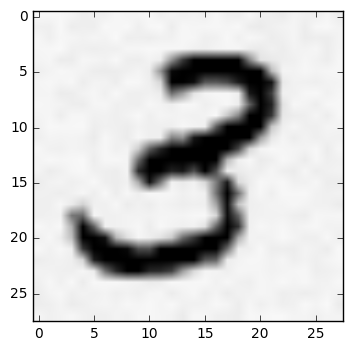}
    \includegraphics[scale = .12]{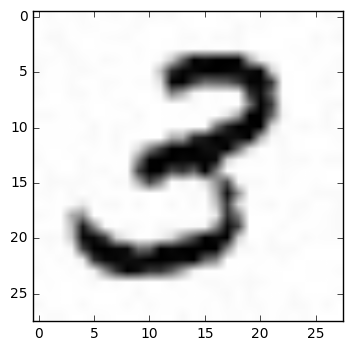}
    \includegraphics[scale = .12]{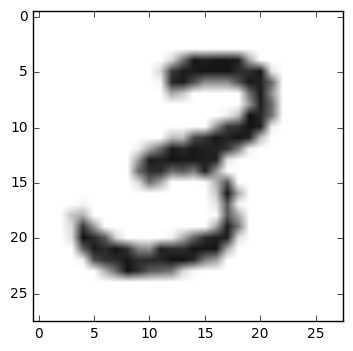}
    \hspace{0.5cm}
    \includegraphics[scale = .12]{three_none.png}
    \includegraphics[scale = .12]{three_checkerboard.png}
    \includegraphics[scale = .12]{three_bothshrink.png}
    \includegraphics[scale = .12]{three_negnoise.png}
    \end{center}
\caption{ Sample MNIST image with each of the corruptions applied to it. Background Corruption Set \& Shrink Corruption Set (top). Pixel Corruption Set \& Mixed Corruption Set (bottom).}\label{fig:corruptions1}   
\end{figure}

\section{Experiments}

\subsection{Robust Classification with Neural Networks}\label{sec:neural-nets}

A classic application of our robust optimization framework is classification with neural networks for corrupted or perturbed datasets. We have a data set $Z$ of pairs $(z,y)$ of an image $z\in\Z$ and label $y\in \Y$ that can be corrupted in $m$ different ways which produces data sets $Z_{1},\ldots,Z_{m}$. The hypothesis space $H$ is the set of all neural nets of some fixed architecture and for each possible assignment of weights. We denote each such hypothesis with $h(\cdot;\theta): \Z\rightarrow \Y$ for $\theta\in \R^d$, with $d$ being the number of parameters (weights) of the neural net. If we let $D_i$ be the uniform distribution over each corrupted data set $Z_i$, then we are interested in minimizing the empirical cross-entropy (aka multinomial logistic) loss in the worst case over these different distributions $D_i$. The latter is a special case of our robust statistical learning framework from Section \ref{sec:statistical}. 

Training a neural network is a non-convex optimization problem and we have no guarantees on its performance. We instead assume that for any given  distribution $D$ over pairs $(z,y)$ of images and labels and for any loss function $\ell(h(z;\theta),y)$, training a neural net with \emph{stochastic gradient descent} run on images drawn from $D$ can achieve an $\alpha$ approximation to the optimal expected loss, i.e. $\min_{\theta\in \R^d}\E_{(z,y)\sim D}\left[\ell(h(z;\theta),y)\right]$. Notice that this implies an $\alpha$-approximate Bayesian Oracle for the corrupted dataset robust training problem: for any distribution $\vec{w}$  over the different corruptions $[m]$, the Bayesian oracle asks to give an $\alpha$-approximation to the minimization problem:
\begin{align}
\min_{\theta\in \R^d} \sum_{i=1}^{m} \vec{w}[i]\cdot \E_{(z,y)\sim D_i}\left[\ell(h(z;\theta),y)\right]
\end{align}
The latter is simply another expected loss problem with distribution over images being the mixture distribution defined by first drawing a corruption index $i$ from $\vec{w}$ and then drawing a corrupted image from distribution $D_i$. Hence, our oracle assumption implies that SGD on this mixture is an $\alpha$-approximation. By linearity of expectation, an alternative way of viewing the Bayesian oracle problem is that we are training a neural net on the original distribution of images, but with loss function being the weighted combination of loss functions $\sum_{i=1}^m \vec{w}[i]\cdot \ell(h(c_i(z);\theta),y)$, where $c_i(z)$ is the $i$-th corrupted version of image $z$. In our experiments we implemented both of these interpretations of the Bayesian oracle, which we call the \textit{Hybrid Method} and \textit{Composite Method}, respectively, when designing our neural network training scheme (see Figure \ref{fig:hybrid} and Figure \ref{fig:composite} in Appendix \ref{sec:appendix-neural-nets}). 
Finally, because we use the cross-entropy loss, which is convex in the prediction of the neural net, we can also apply Theorem \ref{thm:improper-learning} to get that the ensemble neural net, which takes the average of the predictions of the neural nets created at each iteration of the robust optimization, will also achieve good worst-case loss (we refer to this as \textit{Ensemble Bottleneck Loss}).

\paragraph{Experiment Setup.}
We use the MNIST handwritten digits data set containing $55000$ training images, $5000$ validation images, and $10000$ test images, each image being a $28 \times 28$ pixel grayscale image. The intensities of these $576$ pixels (ranging from $0$ to $1$) are used as input to a neural network that has $1024$ nodes in its one hidden layer. The output layer uses the softmax function to give a distribution over digits $0$ to $9$. The activation function is ReLU and the network is trained using Gradient Descent with learning parameter $0.5$ through $500$ iterations of mini-batches of size $100$.

In general, the corruptions can be any black-box corruption of the image. In our experiments, we consider four four types of corruption ($m=4$). See Appendix \ref{sec:appendix-neural-nets} for details about corruptions.

\paragraph{Baselines.} We consider three baselines: (i) \textit{Individual Corruption}: for each corruption type $i \in [m]$, we construct an oracle that trains a neural network using the training data perturbed by corruption $i$, and then returns the trained network weights as $\theta_t$, for every $t=1,\ldots,T$. This gives $m$ baselines, one for each corruption type; (ii) \textit{Even Split}: this baseline alternates between training with different corruption types between iterations. In particular, call the previous $m$ baseline oracles $O_1,...,O_m$. Then this new baseline oracle will produce $\theta_t$ with $O_{i+1}$, where $i \equiv t \mod m$, for every $t=1,...,T$; (iii) \textit{Uniform Distribution}: This more advanced baseline runs the robust optimization scheme with the Hybrid Method (see Appendix), but without the distribution updates. Instead, the distribution over corruption types is fixed as the discrete uniform $[\frac{1}{m},...,\frac{1}{m}]$ over all $T$ iterations. This allows us to check if the multiplicative weight updates in the robust optimization algorithm are providing benefit.

\paragraph{Results.}

\begin{figure}[t]
\begin{center}
    \includegraphics[scale = .35]{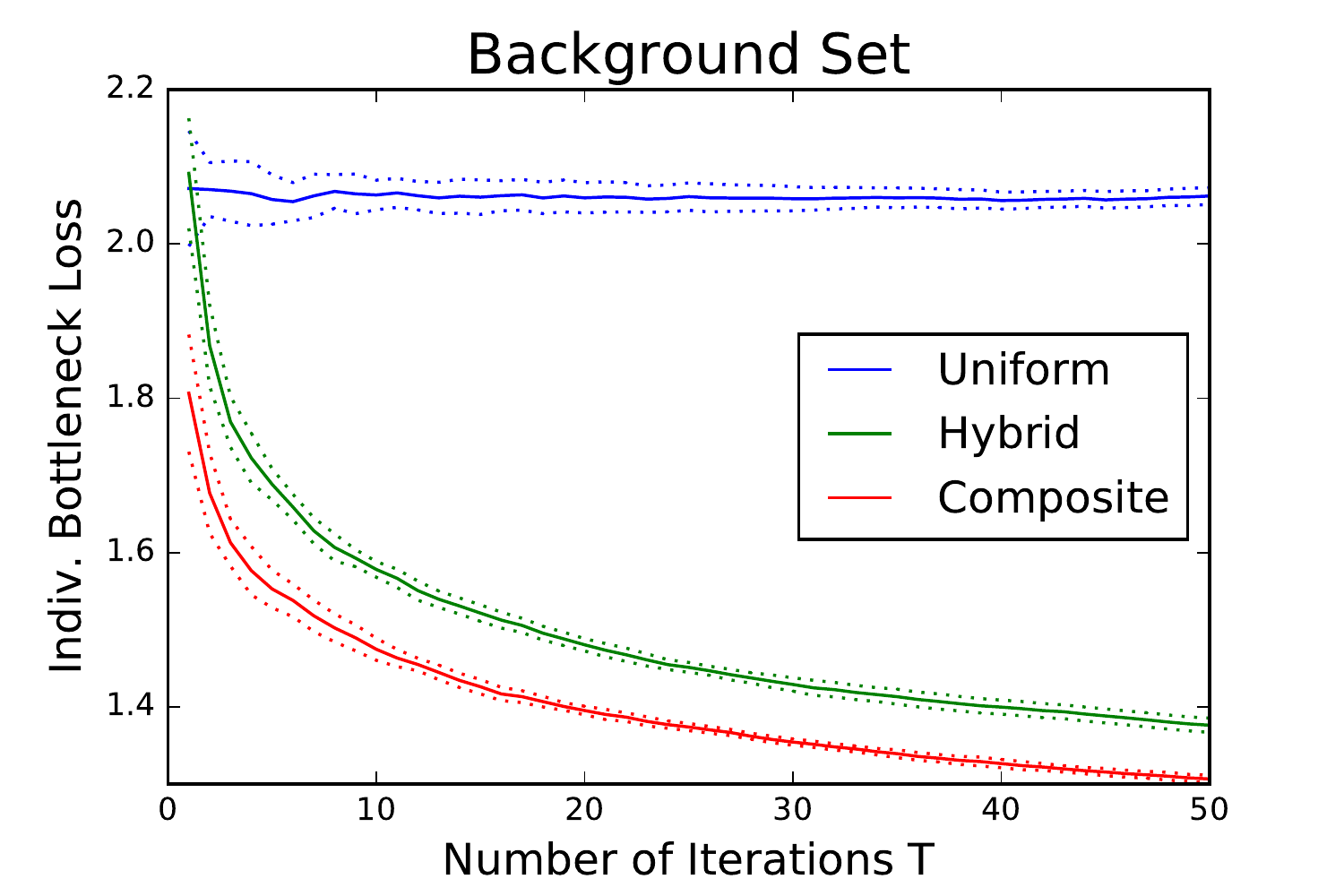}
    \includegraphics[scale = .35]{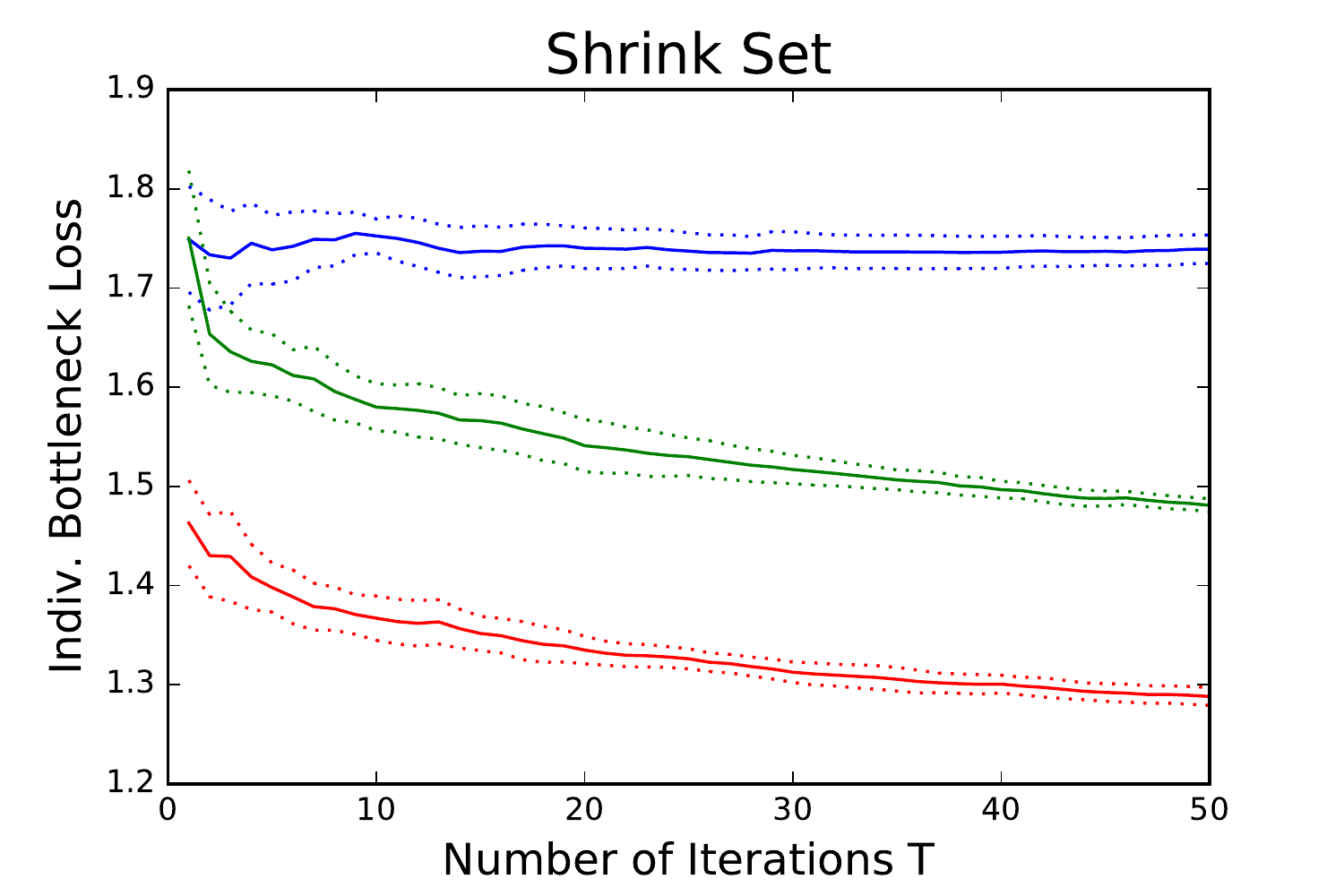}
    \includegraphics[scale = .35]{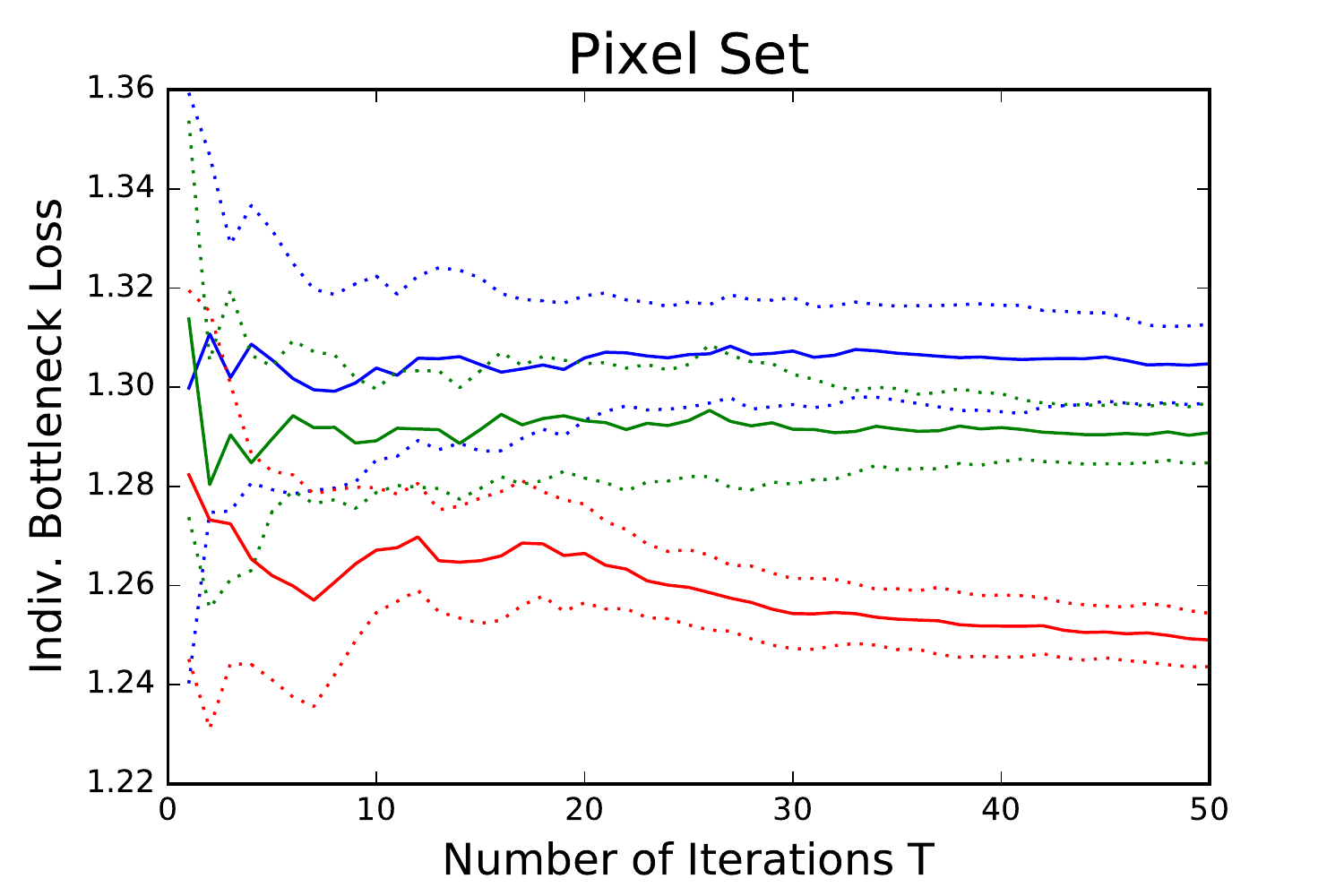}
    \includegraphics[scale = .35]{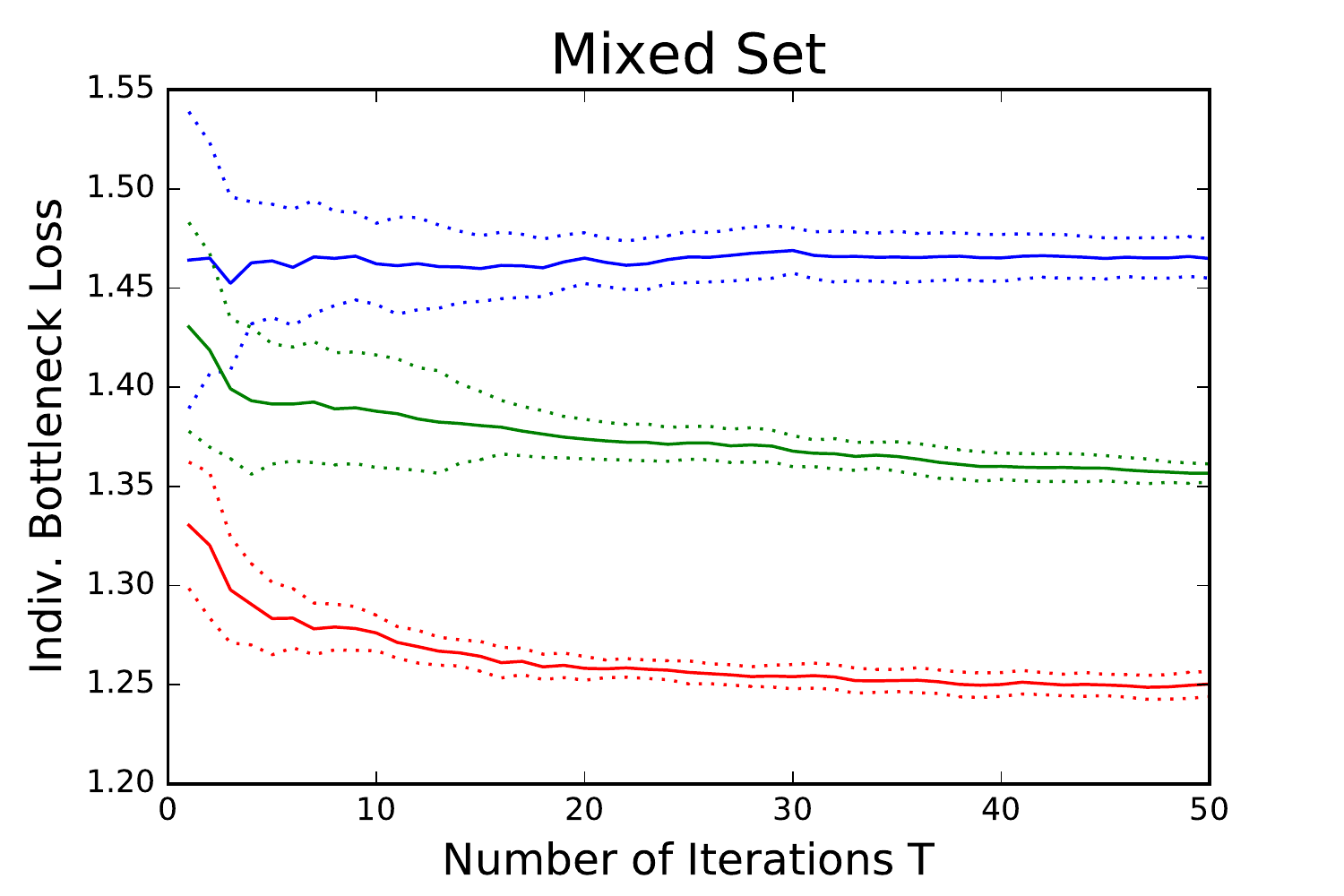}
    \end{center}
    \caption{\footnotesize{Comparison of methods, showing mean of $10$ independent runs and a $95\%$ confidence band. The criterion is \textit{Individual Bottleneck Loss}: $\min_{[m]} E_{\theta \sim P} \left[\ell(h(z;\theta),y)\right]$, where $P$ is uniform over all solutions $\theta_i$ for that method. Baselines (i) and (ii) are not shown as they produce significantly higher loss (see Appendix). }}\label{fig:neuralnetworks} 
\end{figure}

The Hybrid and Composite Methods produce results far superior to all three baseline types, with differences both substantial in magnitude and statistically significant. The more sophisticated Composite Method outperforms the Hybrid Method. Increasing $T$ improves performance, but with diminishing returns--largely because for sufficiently large $T$, the distribution over corruption types has moved from the initial uniform distribution to some more optimal $\textit{stable distribution}$ (see Appendix for details). All these effects are consistent across the 4 different corruption sets tested. The Ensemble Bottleneck Loss is empirically much smaller than Individual Bottleneck Loss. For the best performing algorithm, the Composite Method, the mean Ensemble Bottleneck Loss (mean Individual Bottleneck Loss) with $T=50$ was 0.34 (1.31) for Background Set, 0.28 (1.30) for Shrink Set, 0.19 (1.25) for Pixel Set, and 0.33 (1.25) for Mixed Set. Thus combining the $T$ classifiers obtained from robust optimization is practical for making predictions on new data.
\subsection{Robust Influence Maximization}

We apply the results of Section~\ref{sec:submodular} to the robust influence maximization problem.
Given a directed graph $G=(V,E)$, 
the goal is to pick a \emph{seed set} $S$ of $k$ nodes that maximize an influence function $f_G(S)$, where $f_G(S)$ is the expected number of individuals influenced by opinion of the members of $S$.  We used $f_G(S)$ to be the number of nodes reachable from $S$ (our results extend to other models).


%

In robust influence maximization, the goal is to 
maximize influence in the worst-case (\textit{Bottleneck Influence}) over $m$ functions $\{f_1,\ldots,f_m\}$, corresponding to $m$ graphs $\{G_1,\ldots,G_m\}$, for some fixed seed set of size $k$.
This is a special case of robust submodular maximization after rescaling to $[0,1]$. 

\paragraph{Experiment Setup.}  Given a base directed graph $G(V,E)$, we produce $m$ graphs $G_i=(V,E_i)$ by randomly including each edge $e\in E$ with some probability $p$. We consider two base graphs and two sets of parameters for each: (i) The \emph{Wikipedia Vote Graph}~\cite{snap}. In Experiment $A$, the parameters are $|V|=7115$, $|E|=103689$, $m=10$, $p=0.01$ and $k=10$. In Experiment $B$, change $p=0.015$ and $k=3$. (ii) The \emph{Complete Directed Graph} on $|V|=100$ vertices. In Experiment $A$, the parameters are $m=50$, $p=0.015$ and $k=2$. In Experiment $B$, change $p=0.01$ and $k=4$.

\paragraph{Baselines.} We compared our algorithm (Section \ref{sec:submodular}) to three baselines: (i)
  \textit{Uniform over Individual Greedy Solutions}: Apply greedy maximization (Algorithm \ref{alg:greedy}) on each graph separately, to get solutions $\{S_1^{g},\ldots,S_m^g\}$. Return the uniform distribution over these solutions; (ii) \textit{Greedy on Uniform Distribution over Graphs}: Return the output of greedy submodular maximization (Algorithm \ref{alg:greedy}) on the uniform distribution over influence functions.  This can be viewed as maximizing expected influence; (iii)
  \textit{Uniform over Greedy Solutions on Multiple Perturbed Distributions}: Generate $T$ distributions $\{\vec{w}_1^*,\ldots,\vec{w}_T^*\}$ over the $m$ functions, by  randomly perturbing the uniform distribution.  Perturbation magnitudes were chosen s.t. $\vec{w}_t^*$ has the same expected $\ell_1$ distance from uniform as the distribution returned by robust optimization at iteration $t$.
  

\paragraph{Results.}

\begin{figure}[t]
\begin{center}
    \includegraphics[scale = .35,angle=0,origin=c]{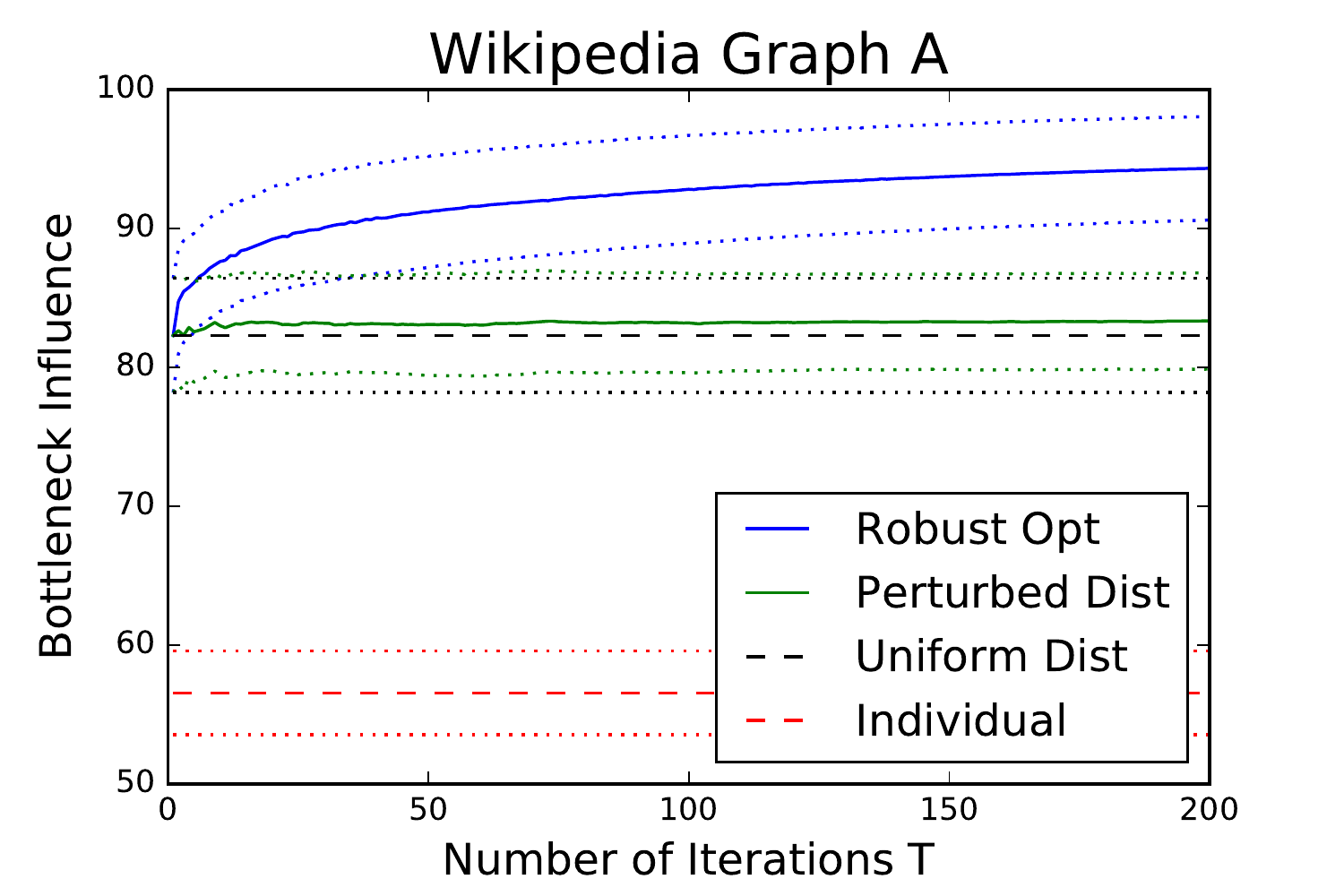}
    \includegraphics[scale = .35,angle=0,origin=c]{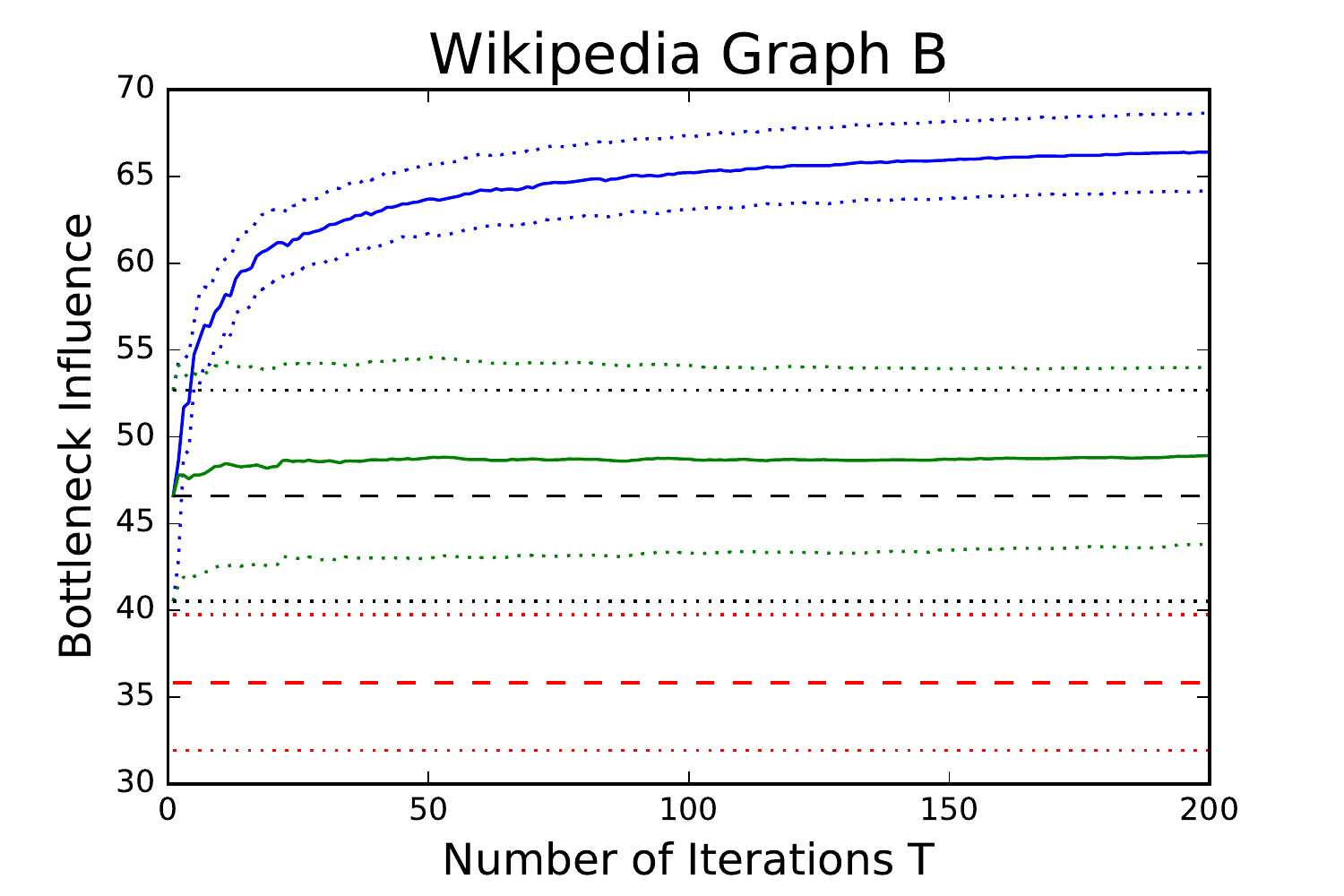}
    \includegraphics[scale = .35,angle=0,origin=c]{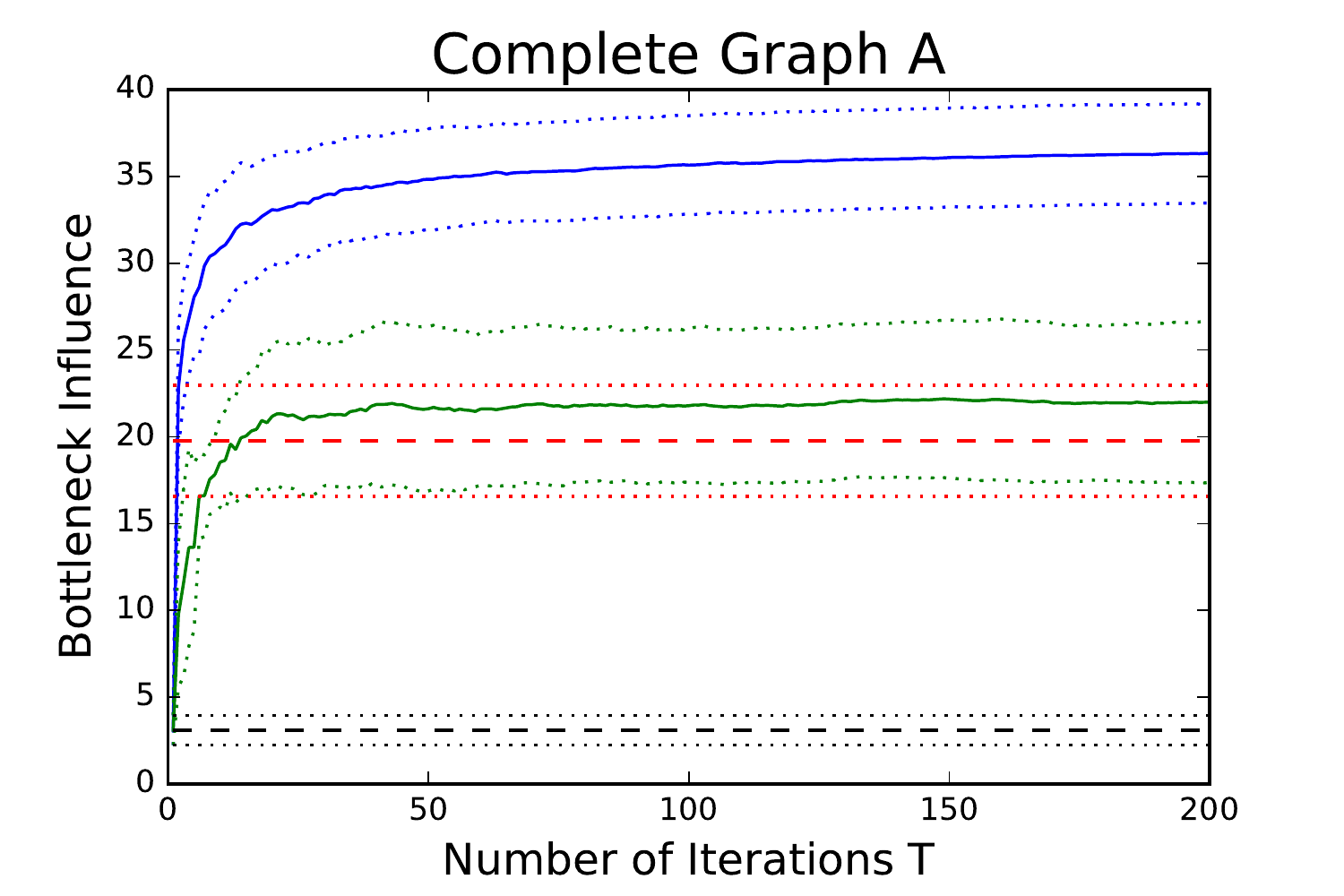}
    \includegraphics[scale = .35,angle=0,origin=c]{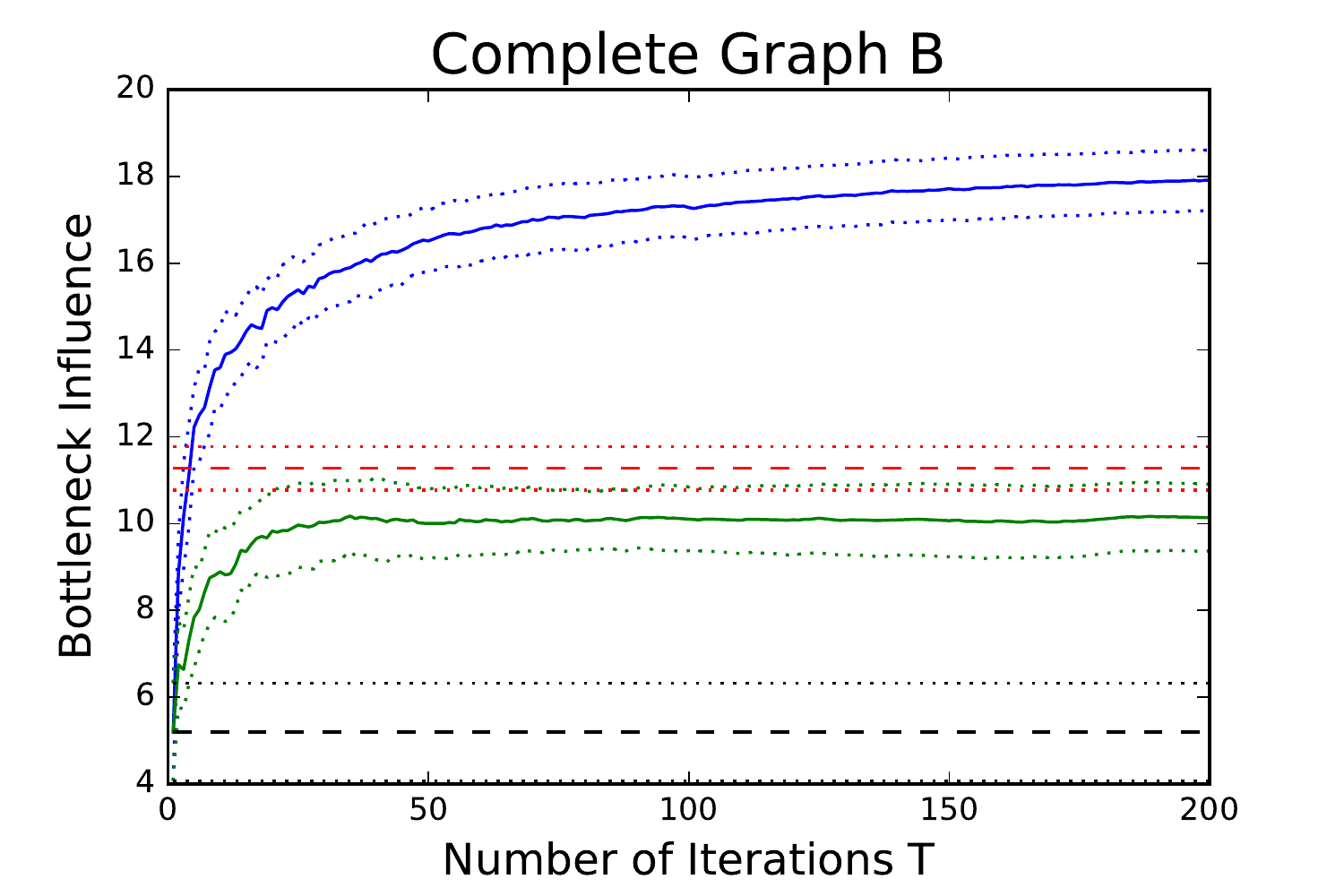}
    \caption{\footnotesize{Comparison for various $T$, showing mean Bottleneck Influence and $95\%$ confidence on $10$ runs.}}
    \end{center}
\end{figure}

%


For both graph experiments, robust optimization outperforms all baselines on Bottleneck Influence; the difference is statistically significant as well as large in magnitude for all $T>50$. 
Moreover, the individual seed sets generated at each iteration $t$ of robust optimization themselves achieve empirically good influence as well (see Appendix for details).

\bibliographystyle{plain}
\bibliography{bib-AGT}

\newpage
\setcounter{page}{1}

 \begin{center}
 \bf \Large Supplementary material for \\ ``\mytitle''
 \end{center}
\begin{appendix}
\section{Faster Convergence to Approximate Solution}
\label{app:faster}

\begin{theorem}[Faster Convergence]\label{thm:faster}
Given access to an $\alpha$-approximate distributional oracle, Algorithm \ref{alg:main} with some parameter $\eta$ computes a distribution $\Po$ over solutions, defined as a uniform distribution over a set $\{x_1,\ldots,x_T\}$, such that
\begin{equation}
\max_{i\in [m]} \E_{x\sim \Po}\left[L_i(x)\right] \leq \alpha(1+\eta) \tau + \frac{\log(m)}{\eta T}.
\end{equation}
In the case of robust reward maximization, the reward version of Algorithm \ref{alg:main} computes a distribution $\Po$ such that:
\begin{equation}
\max_{i\in [m]} \E_{x\sim \Po}\left[L_i(x)\right] \geq \alpha(1-\eta) \tau - \frac{\log(m)}{\eta T}.
\end{equation}
\end{theorem} 
\begin{proof}
We present the case of lossses as the result for the case of rewards follows along similar lines. The proof follows similar lines as that of Theorem \ref{thm:distributionally-robust}. The main difference is that we use a stronger property of the Exponential Weight Updates algorithm. In particular it is known that the regret of EWU, when run on a sequence of rewards that lie in $[-1,1]$ is at most \cite{Arora2012}:
\begin{equation}
\epsilon(T) = \eta \sum_{t=1}^T \E_{I\sim \vec{w}_t}\left[L_I(x_t)^2\right]+ \frac{\log(m)}{\eta T} \leq  \eta \sum_{t=1}^T \E_{I\sim \vec{w}_t}\left[L_I(x_t)\right]+ \frac{\log(m)}{\eta T}
\end{equation}
where the second inequality follows from the fact that $L_i(x)\in [0,1]$. Thus, by the definition of regret, we can write:
\begin{equation}
(1+\eta) \frac{1}{T}\sum_{t=1}^T \E_{I\sim \vec{w}_t}\left[L_I(x_t)\right] \geq \max_{i\in [m]} \frac{1}{T}\sum_{t=1}^T L_i(x_t) - \frac{\log(m)}{\eta T}.
\end{equation}
Combining the above with the guarantee of the distributional oracle we have
\begin{align*}
\tau = \min_{x\in \X}\max_{i\in [m]}  L_i(x) \geq~& \min_{x\in X}\frac{1}{T}\sum_{t=1}^T \E_{I\sim \vec{w}_t}\left[L_I(x)\right]\\
\geq~& \frac{1}{T}\sum_{t=1}^T \min_{x\in X}\E_{I\sim \vec{w}_t}\left[L_I(x)\right]\\
\geq~& \frac{1}{T}\sum_{t=1}^T \frac{1}{\alpha}\cdot \E_{I\sim \vec{w}_t}\left[L_I(x_t)\right] \tag{By oracle guarantee for each $t$}\\
\geq~& \frac{1}{\alpha(1+\eta)}\cdot \left(\max_{i\in [m]} \frac{1}{T}\sum_{t=1}^T L_i(x_t) - \frac{\log(m)}{\eta T}\right).
\tag{By regret of adversary}
\end{align*}
Thus, if we define with $\Po$ to be the uniform distribution over $\{x_1,\ldots,x_T\}$, then we have derived:
\begin{equation}
\max_{i\in [m]} \E_{x\sim \Po}\left[L_i(x) \right] \leq \alpha(1+\eta)\tau + \frac{\log(m)}{\eta T}
\end{equation}
as required.
\end{proof}

\section{Robust Optimization with Infinite Loss Sets}
\label{app:infinite}
We now extend our main results to the case where the uncertainty about the loss function is more general. In particular, we allow for sets of possible losses $\Ell$ that are not necessary finite. In particular, the loss function depends on a parameter $w\in \W$ that is unknown and which could take any value in a set $\W$. The loss of the learner is a function $L(x,w)$ of both his action $x\in \X$ and this parameter $w\in \W$, and the form of the function $L$ is known. Hence, the set of possible losses is defined as:
\begin{equation}
\Ell = \{L(\cdot,w): w\in \W\}
\end{equation}

Our goal is to find some $x\in \X$ that achieves low loss in the worst-case over loss functions in $\Ell$. For $x \in \X$, write $g(x) = \max_{w\in \W} L(x,w)$ for the worst-case loss of $x$. The minimax optimum is
\begin{equation}
\tau = \min_{x\in \X}g(x) = \min_{x\in \X}\max_{w\in \W} L(x,w).
\end{equation}
Our goal in $\alpha$-approximate robust optimization is to find $x$ such that $g(x) \leq \alpha \tau$. Given a distribution $\Po$ over solutions $\X$, write $g(\Po) = \max_{w \in \W}\E_{x \sim \Po}[L(x,w)]$ for the worst-case expected loss of a solution drawn from $\Po$.  The goal of \emph{improper robust optimization}: find a distribution $\Po$ over solutions $\X$ such that $g(\Po) \leq \alpha \tau$.

We will make the assumption that $L(x,w)$ is concave in $w$, $1$-Lipschitz with respect to $w$ and that the set $\W$ is convex. The case of finite losses that we considered in the main text is a special case where the space $\W$ is the simplex on $m$ coordinates, and where: $L(x,w) = \sum_{i=1}^m w[i] \cdot L_i(x)$.

We will also assume that we are given access to an approximate Bayesian oracle, which finds a solution $x \in \X$ that approximately minimizes a given distribution over loss functions:
\begin{defn}[$\alpha$-Approximate Bayesian Oracle] Given a choice of $w\in \W$, the oracle $M(w)$ computes an $\alpha$-approximate solution $x^* = M(w)$ to the known parameter problem, i.e.:
\begin{equation}
L(x^*,w) \leq \alpha\min_{x\in \X} L(x,w)
\end{equation}
\end{defn}

\subsection{Improper Robust Optimization with Oracles}

We first show that, given access to an $\alpha$-approximate distributional oracle, it is possible to efficiently implement improper $\alpha$-approximate robust optimization, subject to a vanishing additive loss term. The algorithm is a variant of Algorithm \ref{alg:main}, where we replace the Multiplicative Weight Updates algorithm for the choice of $w_t$ with a projected gradient descent algorithm, which works for any convex set $\W$. To describe the algorithm we will need some notation. First we denote with $\Pi_{\Y}(w)$ to be the projection of $w$ on the set $\Y$, i.e. $\Pi_{\W}(w) = \arg\min_{w^*\in \W} \|w^*-w\|_2^2$. Moreover, $\nabla_y L(x,y)$ is the gradient of function $L(x,y)$ with respect to $y$. 

\begin{algorithm}[t]
\caption{Oracle Efficient Improper Robust Optimization with Infinite Loss Sets} 
\label{alg:main-infinite}
\begin{algorithmic}
	\STATE {\bfseries Input:} A convex set $\Y$ and loss function $L(\cdot,\cdot)$ which defines the set of possible losses $\Ell$ 
	\STATE {\bfseries Input:} Approximately optimal Bayesian oracle $M$
	\STATE {\bfseries Input:} Accuracy parameter $T$ and step-size $\eta$
	\FOR{each time step $t\in [T]$}
   	\STATE Set
	\begin{align}
	\theta_{t} =~& \theta_{t-1} + \nabla_y L(x_t,y_t)\label{eqn:proj-1}\\
	w_t =~& \Pi_{\W}\left(\eta\cdot \theta_t\right) \label{eqn:proj-2}
	\end{align}   	
   	\STATE Set $x_t = M(w_t)$
	\ENDFOR
	\STATE Output the uniform distribution over $\{x_1,\ldots,x_T\}$
\end{algorithmic}
\end{algorithm}

\begin{theorem}\label{thm:distributionally-robust-infinite}
Given access to an $\alpha$-approximate distributional oracle, Algorithm \ref{alg:main-infinite}, with $\eta=\frac{\max_{w\in \W}\|w\|_2}{\sqrt{2T}}$ computes a distribution $\Po$ over solutions, defined as a uniform distribution over a set $\{x_1,\ldots,x_T\}$, such that:
\begin{equation}
\max_{w\in \W} \E_{x\sim \Po}\left[L(x,w)\right] \leq \alpha \tau + \max_{w\in \W}\|w\|_2 \sqrt{\frac{2}{T}}
\end{equation}
\end{theorem} 
\begin{proof}
We can interpret Algorithm~\ref{alg:main} in the following way.  We define a zero-sum game between a learner and an adversary. The learner's action set is equal to $\X$ and the adversaries action set is $W$. The loss of the learner when he picks $x\in \X$ and the adversary picks $w\in \W$ is defined as $L(x,w)$. The corresponding payoff of the adversary is $L(x,w)$.

We will run no-regret dynamics on this zero-sum game, where at every iteration $t=1,\ldots,T$, the adversary will pick a $w_t\in \W$ and subsequently the learner picks a solution $x_t$. We will be using the projected gradient descent algorithm to compute what $w_t$ is at each iteration, as defined in Equations \eqref{eqn:proj-1} and \eqref{eqn:proj-2}. Subsequently the learner picks a solution $x_t$ that is the output of the $\alpha$-approximate Bayesian oracle on the parameter chosen by the adversary at time-step $t$.  That is,
\begin{equation}
x_t = M\left(w_t\right).
\end{equation}

By the regret guarantees of the projected gradient descent algorithm for the adversary, we have that:
\begin{equation}
\frac{1}{T}\sum_{t=1}^T L(x_t,w_t) \geq \max_{w\in \W} \frac{1}{T}\sum_{t=1}^T L(x_t,w) - \epsilon(T)
\end{equation}
for $\epsilon(T) = \max_{w\in\W}\|w\|_2 \sqrt{\frac{2}{T}}$. Combining the above with the guarantee of the distributional oracle we have
\begin{align*}
\tau = \min_{x\in \X}\max_{w\in \W}  L(x,w) \geq~& \min_{x\in X}\frac{1}{T}\sum_{t=1}^T L(x,w_t)\\
\geq~& \frac{1}{T}\sum_{t=1}^T \min_{x\in \X} L(x,w_t)\\
\geq~& \frac{1}{T}\sum_{t=1}^T \frac{1}{\alpha}\cdot L(x_t,w_t) \tag{By oracle guarantee for each $t$}\\
\geq~& \frac{1}{\alpha}\cdot \left(\max_{w\in \W} \frac{1}{T}\sum_{t=1}^T L(x_t,w) - \epsilon(T)\right).
\tag{By no-regret of adversary}
\end{align*}
Thus if we define with $\Po$ to be the uniform distribution over $\{x_1,\ldots,x_T\}$, then we have derived that
\begin{equation}
\max_{w\in \W} \E_{x\sim \Po}\left[L(x,w) \right] \leq \alpha\tau + \epsilon(T)
\end{equation}
as required.
\end{proof}

A corollary of Theorem~\ref{thm:distributionally-robust} is that if the solution space $\X$ is convex and the function $L(x,y)$ is also convex in $x$ for every $y$, then we can compute a single solution $x^*$ that is approximately minimax optimal. 
\begin{corollary}\label{cor:convex-infinite}
If the space $\X$ is a convex space and the function $L(x,y)$ is convex in $x$ for any $y$, then the point $x^* = \frac{1}{T}\sum_{t=1}^T x_t \in \X$, where $\{x_1,\ldots,x_T\}$ are the output of Algorithm \ref{alg:main-infinite}, satisfies:
\begin{equation}
\max_{w\in \W} L(x^*,w) \leq \alpha \tau +  \max_{w\in \W}\|w\|_2 \sqrt{\frac{2}{T}}
\end{equation}
\end{corollary}
\begin{proof}
By Theorem \ref{thm:distributionally-robust-infinite}, we get that if $\Po$ is the uniform distribution over $\{x_1,\ldots,x_T\}$ then
\begin{equation*}
\max_{w\in \W} \E_{x\sim \Po}[L(x,w)] \leq \alpha \tau + \max_{w\in \W}\|w\|_2 \sqrt{\frac{2}{T}}.
\end{equation*}
Since $\X$ is convex, the solution $x^* = \E_{x\sim \Po}[x]$ is also part of $\X$. Moreover, since each $L(x,y)$ is convex in $x$, we have that $\E_{x\sim \Po}[L(x,y)]\geq L(\E_{x\sim \Po}[x],y) = L(x^*,y)$. We therefore conclude
\begin{equation*}
\max_{w\in \W} L(x^*,w) \leq \max_{w\in \W} \E_{x\sim \Po}[L(x,w)] \leq \alpha \tau + \max_{w\in \W}\|w\|_2 \sqrt{\frac{2}{T}}
\end{equation*}
as required.
\end{proof}

Our results for improper statistical learning can also be analogously generalized to this more general loss uncertainty.

\section{NP-Hardness of Proper Robust Optimization}\label{app:NP-hard}

The convexity assumption of Corollary~\ref{cor:convex} is necessary.  In general, achieving any non-trivial ex-post robust solution is computationally infeasible, even when there are only polynomially many loss functions and they are all concave.

\begin{theorem}
\label{thm:hardness.loss}
There exists a constant $c$ for which the following problem is NP-hard.  Given a collection of linear loss functions $\Ell = \{\ell_1, \dotsc, \ell_m\}$ over a ground set $N$ of $d$ elements, and an optimal distributional oracle over feasibility set $\X = \{ S \subset N \colon |S| = k \}$, find a solution $x^* \in \X$ such that 
\[ \max_{\ell \in \Ell} \ell(x^*) \leq \tau + \frac{k}{m}. \]
\end{theorem}
\begin{proof}
We reduce from the set packing problem, in which there is a collection of sets $\{T_1, \dotsc, T_d\}$ over a ground set $\mathcal{U}$ of $m$ elements $\{u_1, \dotsc, u_m\}$, and the goal is to find a collection of $k$ sets that are all pairwise disjoint.  This problem is known to be NP-hard, even if we assume $k < m/4$.

Given an instance of the set packing problem, we define an instance of robust loss minimization as follows.  There is a collection of $m$ linear functions $\Ell = \{\ell_1, \dotsc, \ell_m\}$, and $N$ is a set of $mk+d$ items, say  $\{a_{ij}\}_{i \leq m, j\leq k} \cup \{b_r\}_{r \leq d}$.  The linear functions are given by $\ell_i(a_{ij}) = 1/k$ for all $i$ and $j$, $\ell_i(a_{i'j}) = 0$ for all $i' \neq i$ and all $j$, $\ell_i(b_r) = 2/m$ if $u_i \in T_r$, and $\ell_i(b_r) = 1/km$ if $u_i \not\in T_r$.

We claim that in this setting, an optimal Bayesian oracle can be implemented in polynomial time.  Indeed, let $D$ be any distribution over $\Ell$, and let $\ell_i$ be any function with minimum probability under $D$.  Then the set $S = \{a_{i1}, \dotsc, a_{ik} \}$ minimizes the expected loss under $D$.  This is because the contribution of any given element $a_{ij}$ to the loss is equal to $1/k$ times the probability of $\ell_i$ under $D$, which is at most $1/km$ for the lowest-probability element, whereas the loss due to any element $b_r$ is at least $1/km$.  Thus, since the optimal Bayesian oracle is polytime implementable, it suffices to show NP-hardness without access to such an oracle.

To establish hardness, note that if a set packing exists, then the solution to the robust optimization problem given by $S = \{ b_r \colon T_r \text{ is in the packing } \}$ satisfies $\ell_i(S) \leq 2/m + (k-1)/km < 3/m$.  On the other hand, if a set packing does not exist, then any solution $S$ for the robust optimization problem either contains an element $a_{ij}$ --- in which case $\ell_i(S) \geq 1/k > 4/m$ --- or must contain at least two elements $b_r, b_s$ such that $T_r \cap T_s \neq \emptyset$, which implies there exists some $i$ such that $\ell_i(S) \geq 4/m$.  We can therefore reduce the set packing problem to the problem of determining whether the minimax optimum $\tau$ is greater than $4/m$ or less than $3/m$.  We conclude that it is NP-hard to find any $S^*$ such that $\max_{\ell \in \Ell} \ell(S^*) \leq \tau + 1/m$.
\end{proof}

Similarly, for robust submodular maximization, in order to achieve a non-trivial approximation guarantee it is necessary to either convexify the outcome space (e.g., by returning distributions over solutions) or extend the solution space to allow solutions that are larger by a factor of $\Omega(\log |\F|)$.  This is true even when there are only polynomially many functions to optimize over, and even when they are all linear.

\begin{theorem}
\label{thm:hardness.value}
There exists a constant $c$ for which the following problem is NP-hard.  Given any $\alpha > 0$, and a collection of linear functions $\F = \{f_1, \dotsc, f_m\}$ over a ground set $N$ of $d$ elements, and an optimal distributional oracle over subsets of $N$ of size $k$, find a subset $S^* \subseteq N$ with $|S^*| \leq c k \log(m)$ such that 
\[ \min_{f \in \F} f(S^*) \geq \frac{1}{\alpha} \tau - \frac{1}{\alpha k m}. \]
\end{theorem}

\begin{proof}
We reduce from the set cover problem, in which there is a collection of sets $\{T_1, \dotsc, T_d\}$ over a ground set $\mathcal{U}$ of $m$ elements $\{u_1, \dotsc, u_m\}$, whose union is $\mathcal{U}$, and the goal is to find a collection of at most $k$ sets whose union is $\mathcal{U}$.  There exists a constant $c$ such that it is NP-hard to distinguish between the case where such a collection exists, and no collection of size at most $c k \log(n)$ exists.  

Given an instance of the set cover problem, we define an instance of the robust linear maximization problem as follows.  There is a collection of $m$ linear functions $\F = \{f_1, \dotsc, f_m\}$, and $N$ is a set of $k m + d$ items, say $\{a_{ij}\}_{i \leq m, j \leq k} \cup \{b_r\}_{r \leq d}$.  For each $i \leq m$ and $j \leq k$, set $f_i(a_{ij}) = 1/k$ and $f_i(a_{i'j}) = 0$ for all $i' \neq i$.  For each $i \leq m$ and $r \leq d$, set $f_i(b_r) = 1/km$ if $u_i \in T_r$ in our instance of the set cover problem, and $f_i(b_r) = 0$ otherwise.

We claim that in this setting, an optimal Bayesian oracle can be implemented in polynomial time.  Indeed, let $D$ be any distribution over $\F$, and suppose $f_i$ is any function with maximum probability under $D$.  Then the set $S = \{a_{i1}, \dotsc, a_{ik} \}$ maximizes expected value under $D$.  This is because the value of any given element $a_{ij}$ is at least $1/k$ times the probability of $f_i$ under $D$, which is at least $1/m$, whereas the value of any element $b_r$ is at most $1/km$.  Thus, since the optimal Bayesian oracle is polytime implementable, it suffices to show NP-hardness without access to such an oracle.

To establish hardness, note first that if a solution to the set cover problem exists, then the solution to the robust optimization problem given by $S = \{b_r \colon T_r \text{ is in the cover } \}$ satisfies $f_i(S) \geq 1/km$ for all $i$.  On the other hand, if no set cover of size $k$ exists, then for any solution $S$ to the robust optimization problem there must exist some element $u_i$ such that $u_i \neq T_r$ for every $b_r \in S$, and such that $a_{ij} \neq S$ for all $j$. This implies that $f_i(S) = 0$, and hence $\tau = 0$.
%
We have therefore reduced the set cover problem to distinguishing cases where $\tau \geq 1/km$ from cases where $\tau = 0$.  We conclude that it is NP-hard to find any $S^*$ for which $\min_{f \in F} f(S^*) \geq \frac{1}{\alpha}(\tau - \frac{1}{km})$, for any positive $\alpha$.
\end{proof}


\section{Strengthening the Benchmark}
We now observe that our construction actually competes with a stronger benchmark than $\tau$. In particular, one that allows for distributions over solutions: 
\begin{equation}
\tau^* = \min_{\G\in \Delta(\X)}\max_{i\in [m]} \E_{x\sim \G}[L_i(x)]
\end{equation}
Hence, our assumption is that there exists a distribution $\G$ over solutions $\X$ such that for any realization of the objective function, the expected value of the objective under this distribution over solutions is at least $\tau^*$. 

Now we ask: given an oracle for the distributional problem, can we find a solution for the robust problem that achieve minimum reward at least $\tau^*$. We show that this is possible:
\begin{theorem}\label{thm:distributionally-robust-2}
Given access to an $\alpha$-approximate Bayesian oracle, we can compute a distribution $\Po$ over solutions, defined as a uniform distribution over a set $\{x_1,\ldots,x_T\}$, such that:
\begin{equation}
\max_{i\in [m]} \E_{x\sim \Po}\left[L_i(x)\right] \leq \alpha \tau^*+ \sqrt{\frac{2\log(m)}{T}}
\end{equation}
\end{theorem} 
\begin{proof}
Observe that a distributional oracle for the setting with solution space $\X$ and functions $\Ell=\{L_1,\ldots,L_m\}$ is also a distributional oracle for the setting with solution space $\D=\Delta(\X)$ and functions $\Ell'= \{L_1',\ldots,L_m'\}$, where for any $D\in \D$: $L_j'(D) = \E_{x\sim D}\left[L_j(x)\right]$. Moreover, observe that $\tau^*$ is exactly equal to $\tau$ for the setting with solution space $\D$ and function space $\Ell'$. Thus applying Theorem \ref{thm:distributionally-robust} to that setting we get an algorithm which computes a distribution $\Po'$ over distributions of solutions in $\X$, that satisfies:
\begin{equation}
\max_{i\in [m]} \E_{D\sim \Po'}\left[\E_{x\sim D}[L_j(x)]\right] \leq \alpha \tau^*+\sqrt{\frac{2\log(m)}{T}}
\end{equation} 
Observe that a distribution over distributions of solutions is simply a distribution over solutions, which concludes the proof of the Theorem.
\end{proof}


\section{Experiments}\label{sec:appendix-neural-nets}

\subsection{Hybrid Method}

In order to apply the robust optimization algorithm we need to construct a neural network architecture that facilitates it. In each iteration $t$, such an architecture receives a distribution over corruption types $\textbf{w}_t = [\textbf{w}_t[1],...,\textbf{w}_t[m]]$ and produces a set of weights $\theta_t$.

\begin{figure}[H]
 \begin{center}
    \includegraphics[scale = .25,angle=0,origin=c]{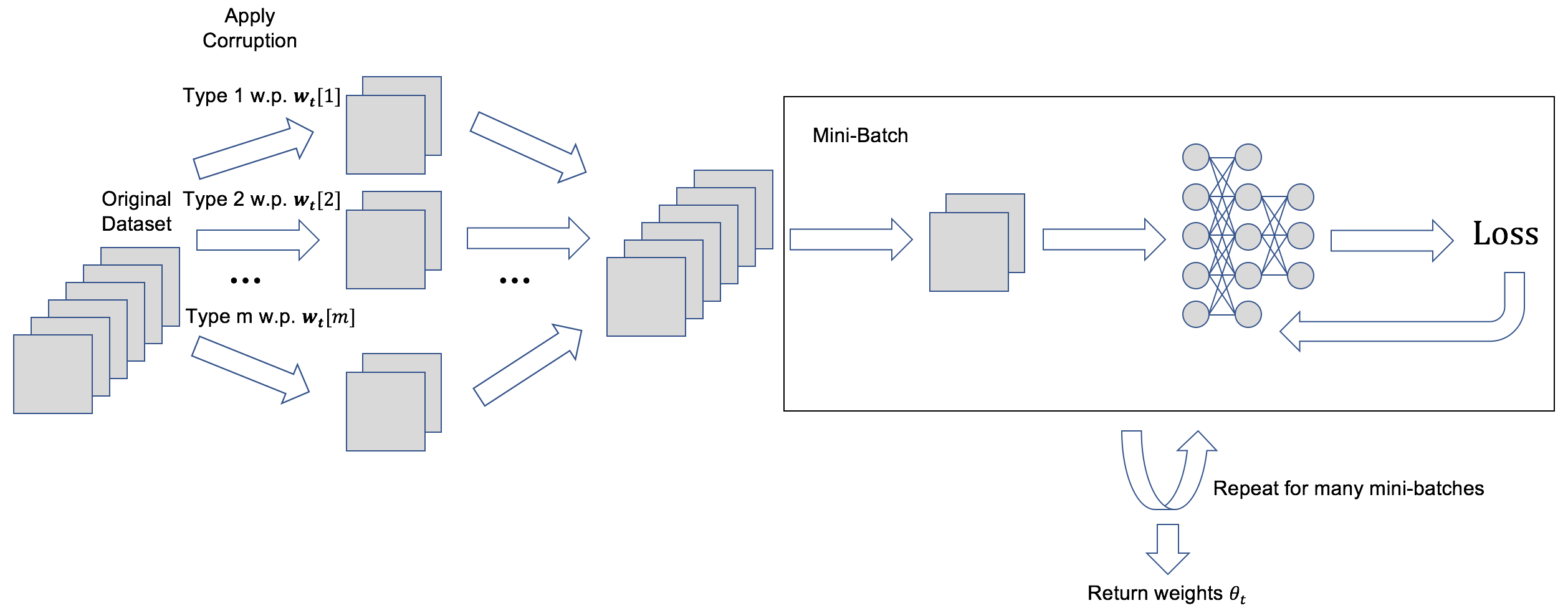}
    \end{center}
\caption{First interpretation of Bayesian oracle, training on a sample of images drawn from the mixture of corruptions.}\label{fig:hybrid}   
\end{figure}

In the Hybrid Method, our first oracle, we take each training data image and perturb it by exactly one corruption, with corruption $i$ being selected with probability $\textbf{w}_t[i]$. Then apply mini-batch gradient descent, picking mini-batches from the perturbed data set, to train a classifier $\theta_t$. Note that the resulting classifier will take into account corruption $i$ more when $\textbf{w}_t[i]$ is larger.

\subsection{Composite Method} 
\begin{figure}[H]
 \begin{center}
    \includegraphics[scale = .2,angle=0,origin=c]{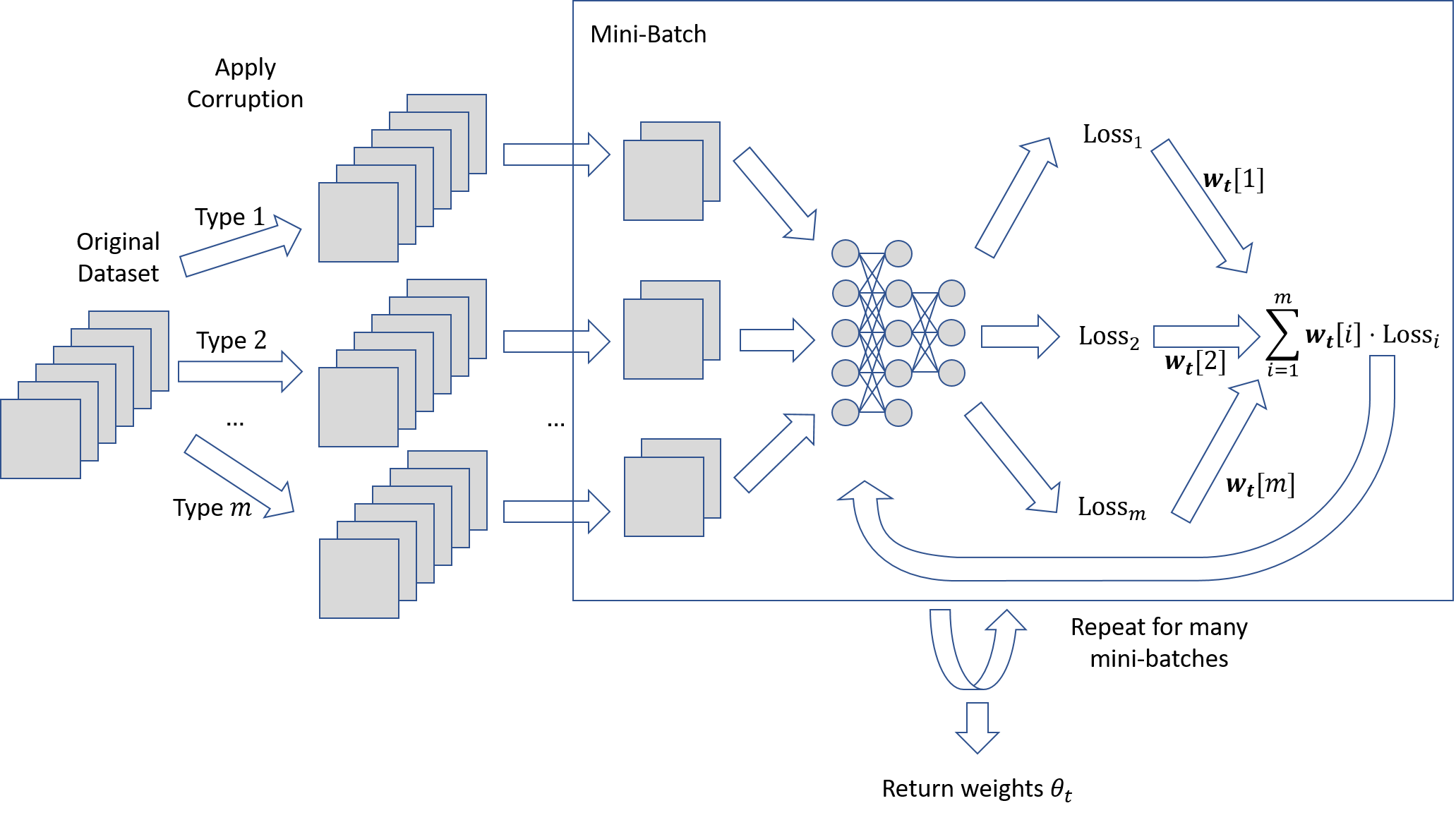}
    \end{center}
    \caption{Second interpretation of Bayesian oracle, by creating $m$ coupled instantiations of the net architecture (one for each corruption type), with the $i$-th instance taking as input the image corrupted with the $i$-th corruption and then defining the loss as the convex combination of the losses from each instance.}\label{fig:composite}
\end{figure}

In the Composite Method, at each iteration, we use $m$ copies of the training data, where copy $i$ has Corruption Type $i$ applied to all training images. The new neural network architecture has $m$ sub-networks, each taking in one of the $m$ training data copies as input. All sub-networks share the same set of neural network weights. During a step of neural network training, a mini-batch is selected from the original training image set, and the corresponding images in each of the $m$ training set copies are used to compute weighted average of the losses $\sum_{i=1}^m \textbf{w}_t[i] Loss_{t,i}$, which is then used to train the weights.

\subsection{Corruption Set Details}

\textbf{Background Corruption Set} consists of images with (i) an unperturbed white background--the original images, (ii) a light gray tint background, (iii) a gradient background, (iv) and a checkerboard background.

\noindent \textbf{Shrink Corruption Set} consists of images with (i) no distortion--the original images, (ii) a $25\%$ shrinkage along the horizontal axis, (iii) a $25\%$ shrinkage along the vertical axis, and (iv) a $25\%$ shrinkage in both axes.

\noindent \textbf{Pixel Corruption Set} consists of images that (i) remain unaltered--the original images, (ii) have $Unif[-0.15,-0.05]$ perturbation added i.i.d. to each pixel, (iii) have $Unif[-0.05,0.05]$ perturbation added i.i.d. to each pixel, and (iv) have $Unif[0.05,0.15]$ perturbation added i.i.d. to each pixel.

\noindent \textbf{Mixed Corruption Set} consists of images that (i) remain unaltered--the original images, and one corruption type from each of the previous three corruption sets (which were selected at random), namely that with (ii) the checkerboard background, (iii) $25\%$ shrinkage in both axes, and (iv) i.i.d. $Unif[-0.15,-0.05]$ perturbation.


\subsection{Neural Network Results}

\begin{table}[H]
\begin{center}
  \begin{tabular}{ | l | c | c | c | c |}
  \hline
     & \textit{Background Set} & \textit{Shrink Set} & \textit{Pixel Set} & \textit{Mixed Set}\\ \hline
    \textit{Best Individual Baseline} & \makecell{\textbf{8.85} \\ \textit{(8.38,9.32)}} & \makecell{\textbf{7.19} \\ \textit{(7.09,7.28)}} & \makecell{\textbf{1.82} \\ \textit{(1.81,1.82)}} & \makecell{\textbf{8.75} \\ \textit{(8.50,9.00)}} \\ \hline
    \textit{Even Split Baseline} & \makecell{\textbf{28.35} \\ \textit{(26.81,29.89)}} & \makecell{\textbf{11.54} \\ \textit{(11.25,11.83)}} & \makecell{\textbf{1.93} \\ \textit{(1.91,1.95)}} & \makecell{\textbf{9.92} \\ \textit{(9.78,10.06)}} \\ \hline
    \textit{Uniform Distribution Baseline} & \makecell{\textbf{2.06} \\ \textit{(2.05,2.08)}} & \makecell{\textbf{1.74} \\ \textit{(1.72,1.76)}} & \makecell{\textbf{1.30} \\ \textit{(1.30,1.31)}} & \makecell{\textbf{1.46} \\ \textit{(1.45,1.47)}} \\ \hline
    \textit{Hybrid Method} & \makecell{\textbf{1.38} \\ \textit{(1.37,1.39)}} & \makecell{\textbf{1.48} \\ \textit{(1.47,1.49)}} & \makecell{\textbf{1.29} \\ \textit{(1.28,1.30)}} & \makecell{\textbf{1.36} \\ \textit{(1.35,1.36)}}\\ \hline
    \textit{Composite Method} & \makecell{\textbf{1.31} \\ \textit{(1.30,1.31)}} & \makecell{\textbf{1.30} \\ \textit{(1.29,1.31)}} & \makecell{\textbf{1.25} \\ \textit{(1.24,1.25)}} & \makecell{\textbf{1.25} \\ \textit{(1.24,1.26)}}\\ \hline
  \end{tabular}
  \end{center}
  \caption{Individual Bottleneck Loss results (mean over $10$ independent runs and a $95\%$ confidence interval for the mean) with $T=50$ on all four Corruption Sets. Composite Method outperforms Hybrid Method, and both outperform baselines, with such differences being statistically significant.} \label{tbl:neuralnetwork}
\end{table}

\subsection{Analysis of Multiplicative Weights Update}

Consider the robust optimization algorithm using the Hybrid and Composite Methods, but parameterizing $\eta$ as $\eta = c \cdot T^{-\gamma}$ (for constant $c = \sqrt{\frac{\log{m}}{2}}$) to alter the multiplicative weights update formula. In this paper, we have been using $\gamma = 0.5 \implies \eta = \frac{c}{\sqrt{T}}$. Lower values of $\gamma$ leads to larger changes in the distribution over corruption types between robust optimization iterations. Here we rerun our experiments from Section \ref{sec:neural-nets} using $\gamma = 0.1$; we did not tune $\gamma$--the only values of $\gamma$ tested were $0.1$ and $0.5$.\footnote{A possible future step would be to use cross-validation to tune $\gamma$ or design an adaptive parameter algorithm for $\gamma$.}

\begin{figure}[H]
 \begin{center}
    \includegraphics[scale = .4,angle=0,origin=c]{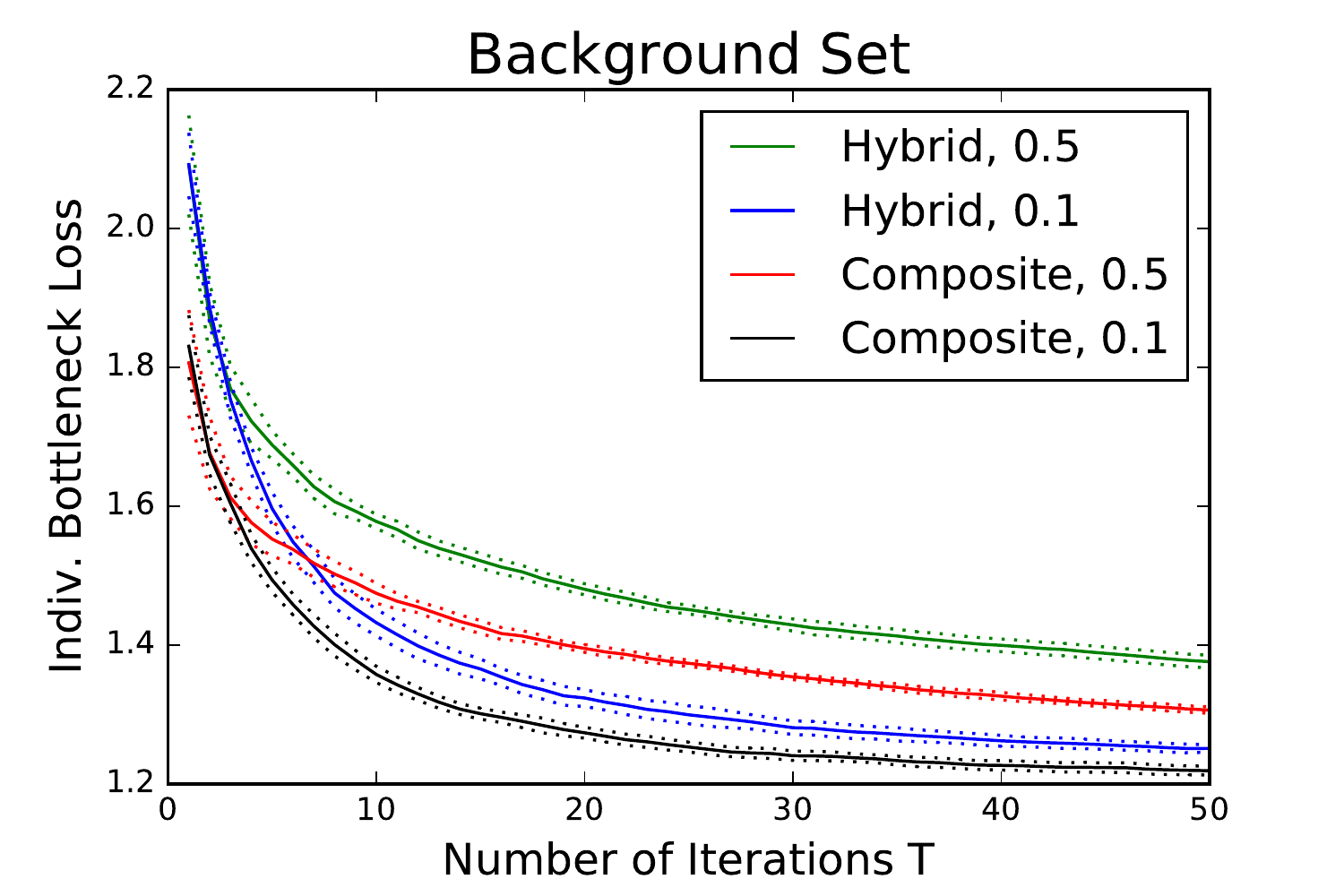}
    \includegraphics[scale = .4,angle=0,origin=c]{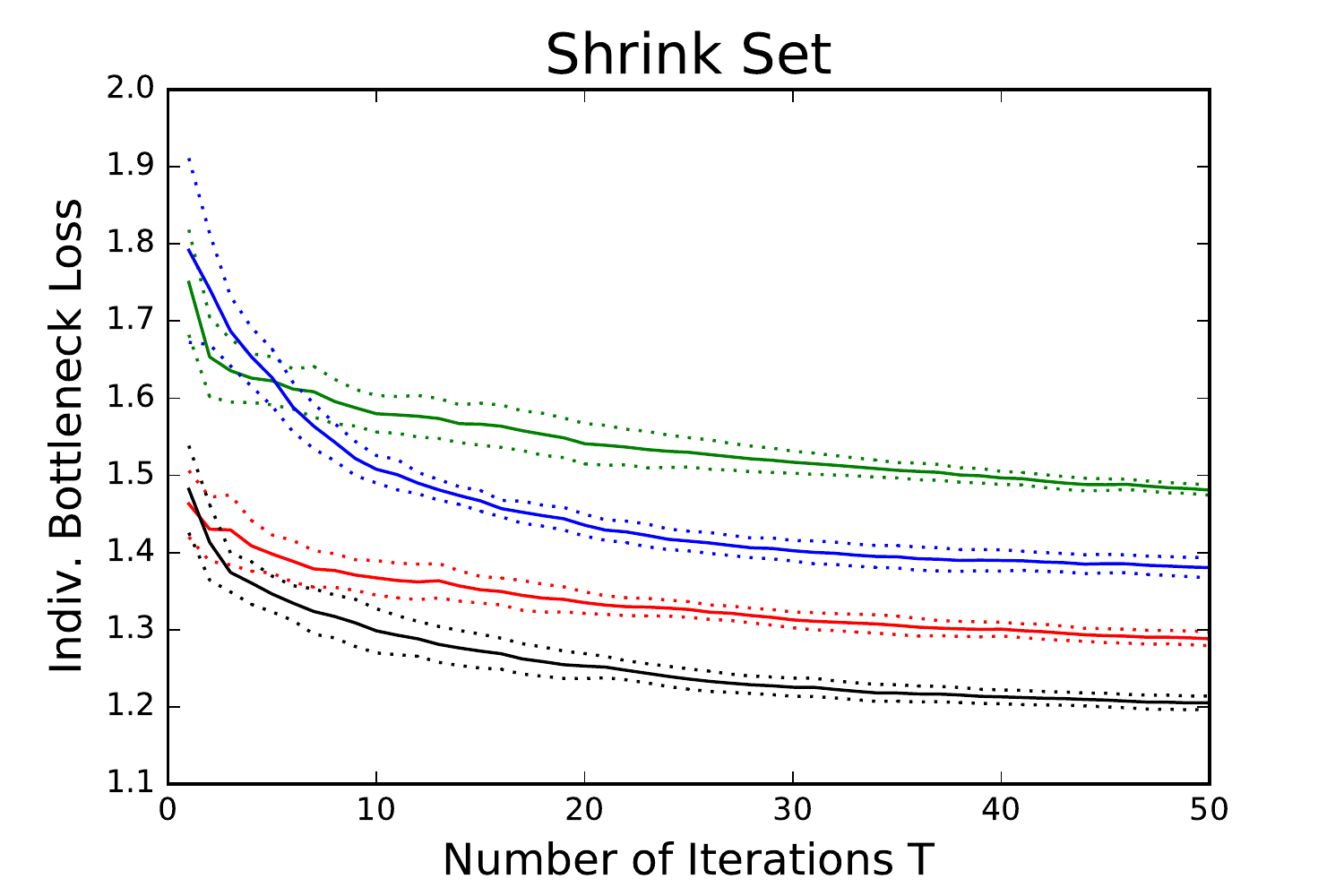}
    \includegraphics[scale = .4,angle=0,origin=c]{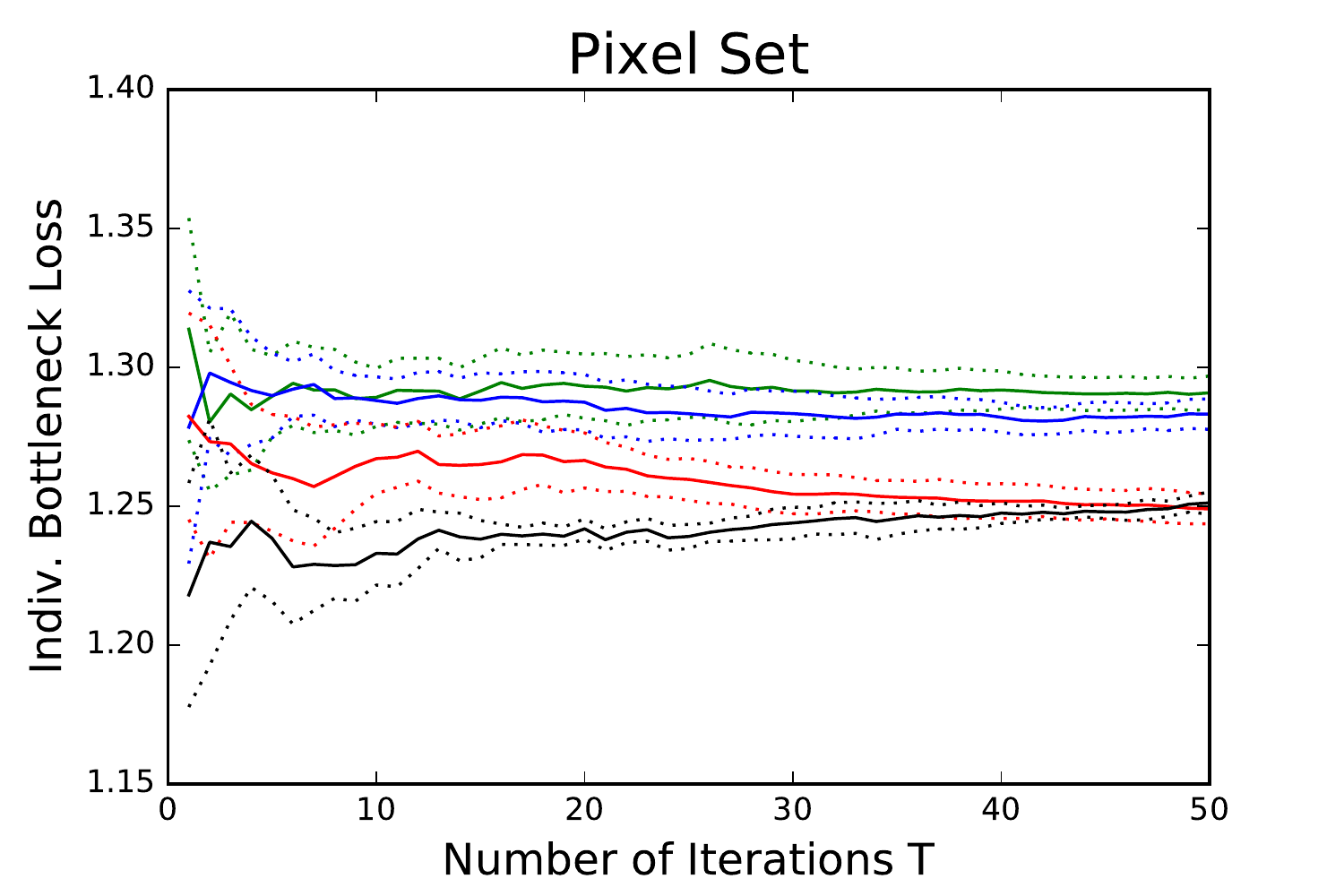}
    \includegraphics[scale = .4,angle=0,origin=c]{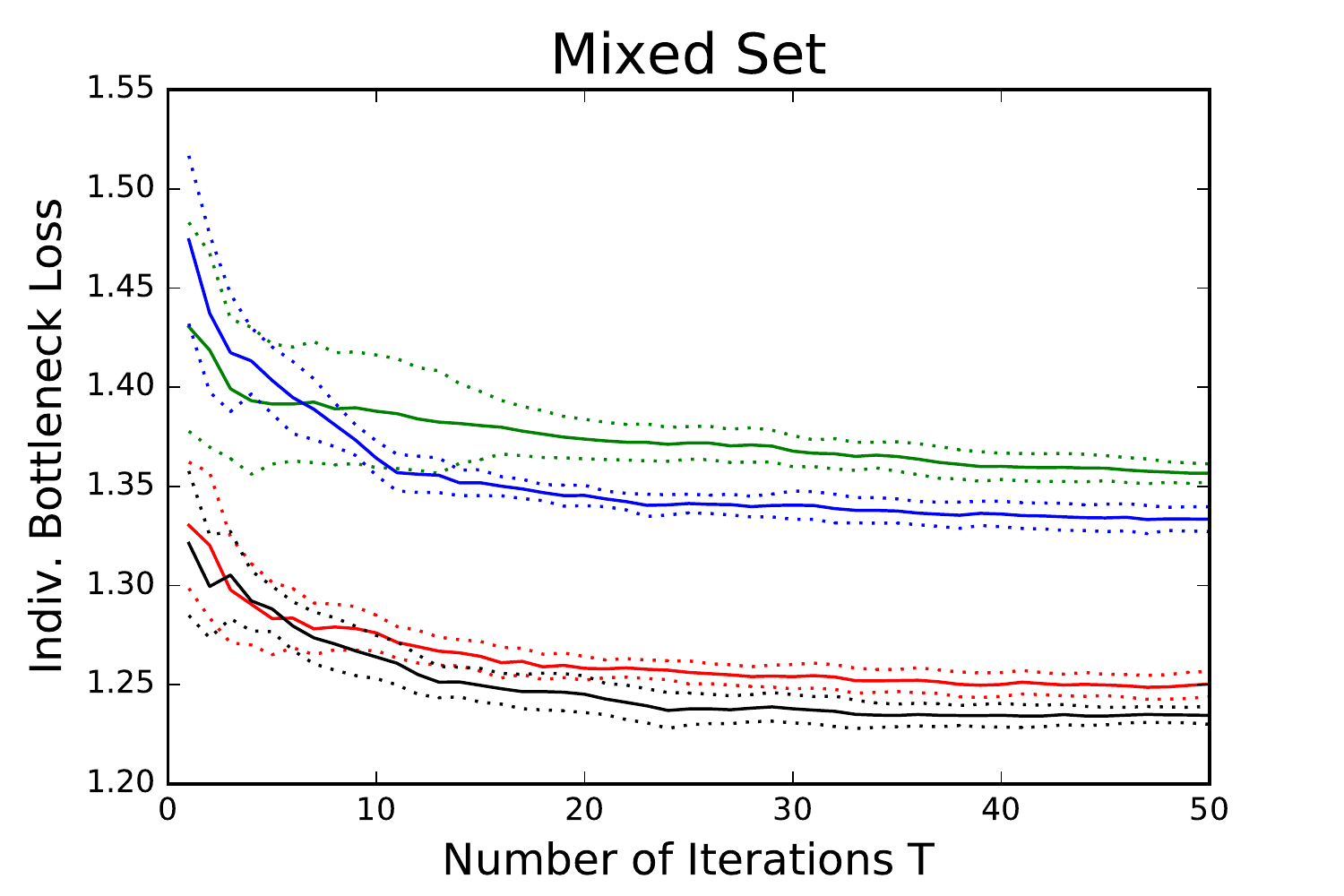}
    \end{center}
    \caption{Comparison of Individual Bottleneck Loss between using $\gamma = 0.5$ vs. $\gamma = 0.1$ in the multiplicative weights update, for both the Hybrid and Composite Methods. The $\gamma = 0.1$ setting yields lower loss.}
    \label{fig:change_gamma}
\end{figure}

The improved performance with $\gamma = 0.1$ compared to $\gamma = 0.5$ is related to an important property of our robust optimization algorithm in practice--namely that $\textbf{w}$ stabilizes for sufficiently large $T$. Over the course of iterations of the algorithm, $\textbf{w}$ moves from the initial discrete uniform distribution to some optimal \textit{stable distribution}, where the stable distribution is consistent across independent runs. The $\gamma = 0.1$ setting yields to better Individual Bottleneck Loss than the $\gamma = 0.5$ setting for finite $T$ because it converges more rapidly to the stable distribution.

\begin{figure}[H]
 \begin{center}
    \includegraphics[scale = .4,angle=0,origin=c]{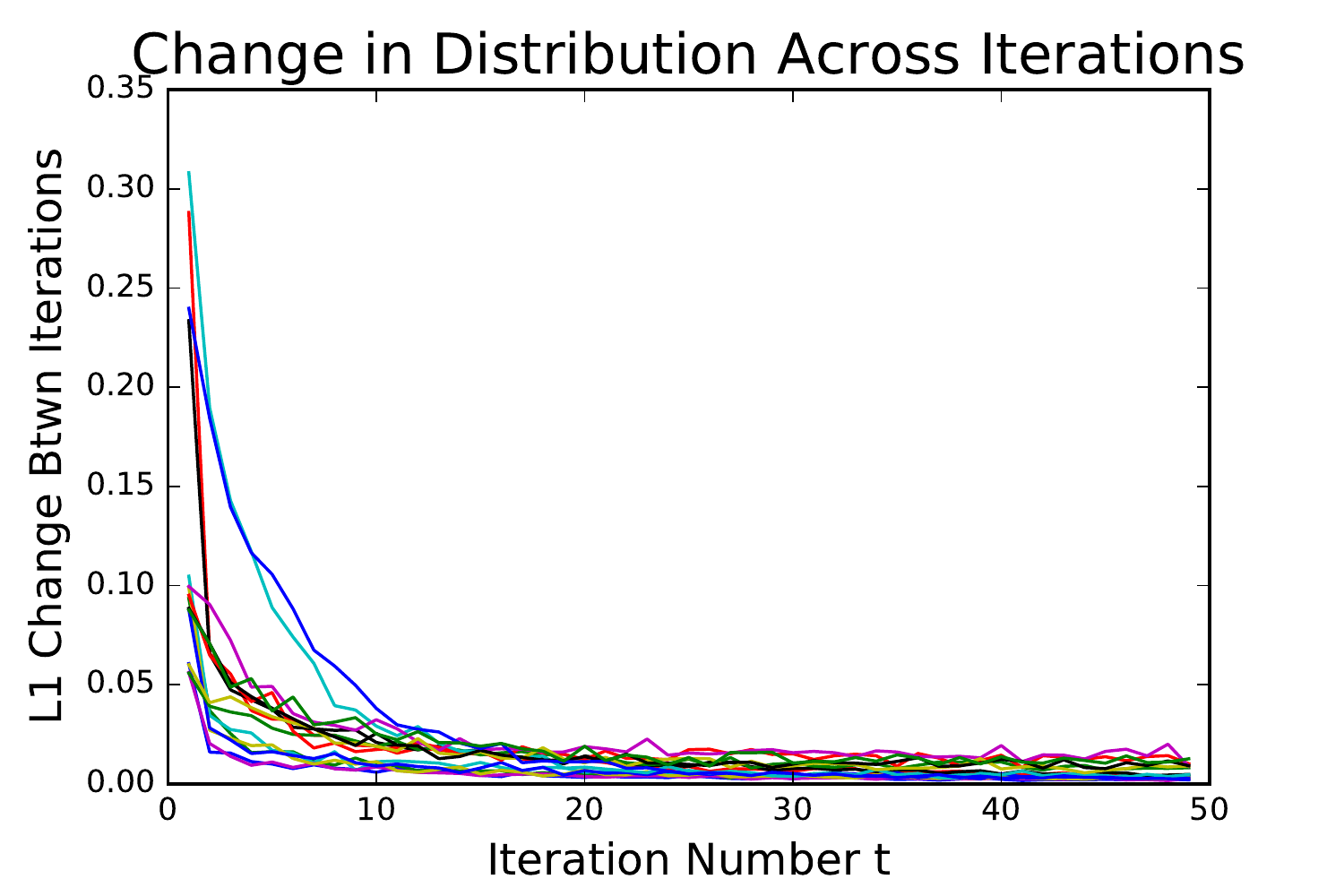}
    \includegraphics[scale = .4,angle=0,origin=c]{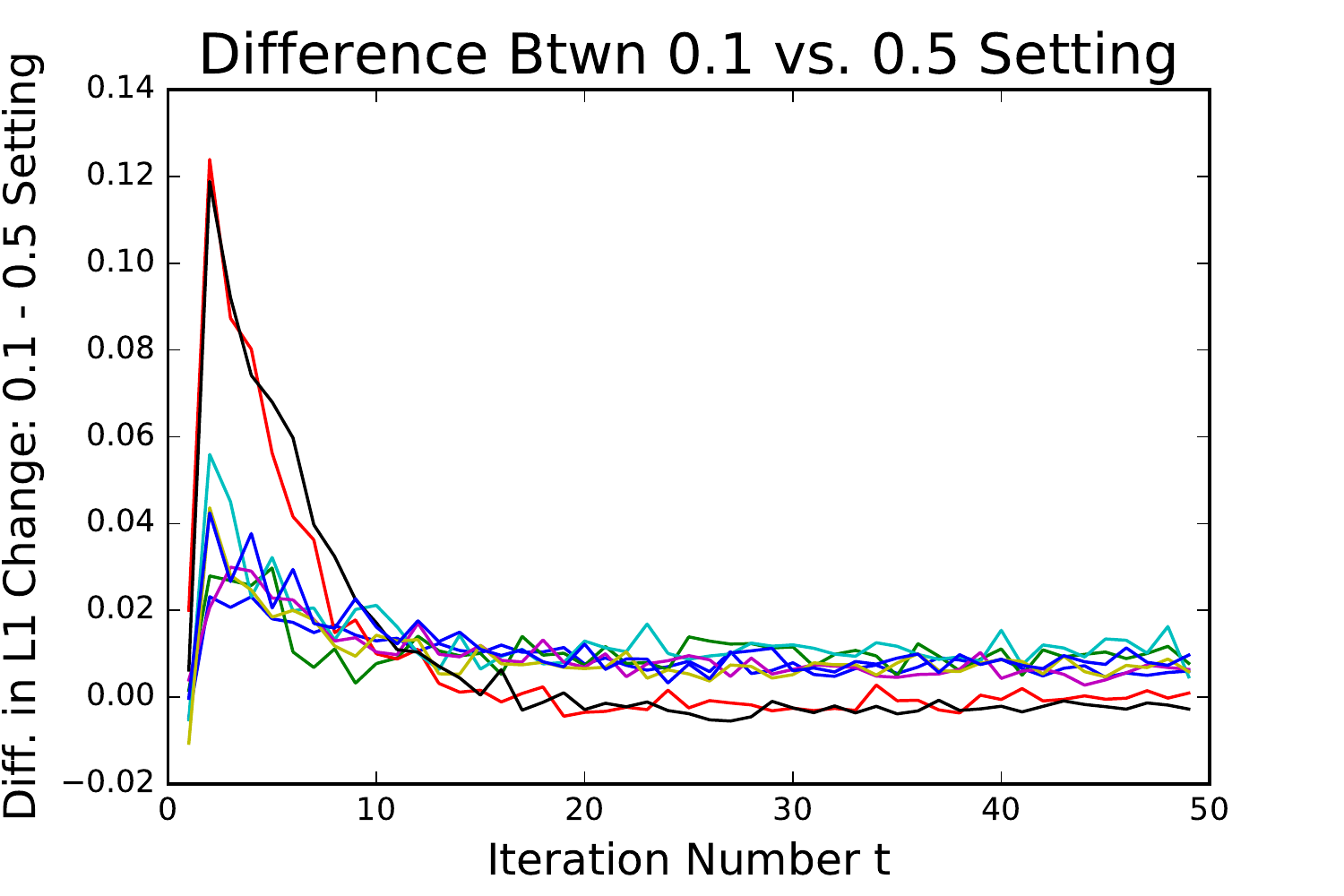}
    \end{center}
    \caption{\textbf{Left:} The amount that the distribution over corruption types $\textbf{w}$ changes between iteration $t$ \& $t+1$ decays rapidly as $t$ increases, and the distribution stabilizes. Plot shows $16$ time series, corresponding to results for each combination of (\{Hybrid, Composite\},\{$\gamma=0.5, \gamma=0.1$\},\{Background, Shrink, Pixel, Mixed\}), using the mean over $10$ runs. \textbf{Right:} The difference between $\gamma = 0.1$ \& $\gamma = 0.5$ in the amount that $\textbf{w}$ changes between iterations. Shows the difference between pairs of time series from the previous figure (thus there are $\frac{16}{2}=8$ time series shown). Values are positive for small $t$ and near $0$ for larger $t$, showing that the $\gamma = 0.1$ setting yields faster changes in $\textbf{w}$ initially, thereby allowing $\textbf{w}$ to more quickly approach the stable distribution.}
    \label{fig:convergence}
\end{figure}

\section{Experiments on Robust Influence Maximization}

\subsection{Influence Results}

\begin{table}[H]
\begin{center}
  \begin{tabular}{ | l | c | c | c | c |}
  \hline
     & \textit{Wikipedia A} & \textit{Wikipedia B} & \textit{Complete A} & \textit{Complete B} \\ \hline
    \textit{Individual Baseline} & \makecell{\textbf{56.56} \\ \textit{(53.55,59.57)}} & \makecell{\textbf{35.84} \\ \textit{(31.93,39.75)}} & \makecell{\textbf{19.77} \\ \textit{(16.57,22.96)}} & \makecell{\textbf{11.27} \\ \textit{(10.77,11.77)}} \\ \hline
    \textit{Uniform Baseline} & \makecell{\textbf{82.30} \\ \textit{(78.19,86.41)}} & \makecell{\textbf{46.60} \\ \textit{(40.53,52.67)}} & \makecell{\textbf{3.10} \\ \textit{(2.24,3.96)}} & \makecell{\textbf{5.20} \\ \textit{(4.07,6.33)}} \\ \hline
    \textit{Perturbed Dist. Baseline} & \makecell{\textbf{83.35} \\ \textit{(79.87,86.82)}} & \makecell{\textbf{48.92} \\ \textit{(43.80,54.03)}} & \makecell{\textbf{21.99} \\ \textit{(17.38,26.61)}} & \makecell{\textbf{10.14} \\ \textit{(9.37,10.91)}} \\ \hline
    \textit{Robust Optimization} & \makecell{\textbf{94.33} \\ \textit{(90.61,98.05)}} & \makecell{\textbf{66.42} \\ \textit{(64.17,68.66)}} & \makecell{\textbf{36.34} \\ \textit{(33.46,39.21)}} & \makecell{\textbf{17.91} \\ \textit{(17.22,18.60)}} \\ \hline
  \end{tabular}
  \end{center}
  \caption{Mean worst-case influence $\min_{i \in [m]} E_{S \sim \Po}[f_i(S)]$ for the solution $\Po$ returned by each method, over $10$ independent runs using $T=200$, and $95\%$ confidence intervals for those means.}
\label{tbl:influence}
\end{table}

Robust Optimization outperforms the baselines, and the differences are statistically significant.\footnote{Claim of statistical significance is based on means of differences between methods, which controls for differences in the $G_i$, rather than differences between means, which are shown in Table \ref{tbl:influence}.}


\subsection{Performance of Single Solutions}

For the Complete Graph $A$ case, it is computationally feasible to obtain the absolute best seed set (via brute force over $\binom{100}{2}$ total possible seed sets), so we can consider the ratio of the best individual seed set generated at some iteration $t$ by robust optimization to the absolute best seed set--that is, $\frac{\max_{S \in \Po} \min_{i \in [m]} f_i(S)}{\max_{S} \min_{i \in [m]} f_i(S)}$. The mean of this ratio over $10$ runs was $0.733$.

For the other three cases, it is not computationally feasible to obtain the absolute best seed set, but we can instead compare the best individual seed set generated by the robust optimization procedure to the Bottleneck Influence value from considering all of $\Po = \{S_1,...,S_T\}$--specifically, the ratio $\frac{\max_{S \in \Po} \min_{i \in [m]} f_i(S)}{\min_{i \in [m]} E_{S \sim \Po} f_i(S)}$. Based on the mean of $10$ runs, this ratio is $0.995$ for Wikipedia $A$, $0.855$ for Wikipedia $B$, and $0.509$ for Complete $B$. The individual seed sets generated by the robust optimization procedure are thus especially good for the Wikipedia Graph; those Wikipedia Graph results are more representative of real graphs, since the Complete Graph has an artificially small number of nodes ($|V|=100$).
\end{appendix}
\end{document}